\documentclass[twoside,11pt]{article}

\usepackage{jmlr2e}
\usepackage{microtype}
\usepackage{graphicx}
\usepackage{subfigure}
\usepackage{parskip}
\usepackage{booktabs} 
\usepackage{amsmath}
\usepackage{amssymb}

 \usepackage{algorithm}
 \usepackage{algorithmic}
\usepackage{bbm}
\usepackage{float}
\usepackage[utf8]{inputenc}

\usepackage{color}

\usepackage{jmlr2e}

\jmlrheading{1}{2000}{1-48}{4/00}{10/00}{Marina Meil\u{a} and Michael I. Jordan}

\ShortHeadings{Black-box Generalization of Machine Teaching}{Xiaofeng and Yaming}
\firstpageno{1}

\begin{document}

\title{Black-box Generalization of Machine Teaching}

\author{\name Xiaofeng Cao\thanks{Preliminary work was done when Xiaofeng Cao was a Research Assistant at AAII, UTS.} \email xiaofeng.cao.uts@gmail.com \\
\addr    School of Artificial Intelligence  \\ Jilin University,  Changchun, 130012, China \\
\name Yaming Guo  \email guoym21@mails.jlu.edu.cn \\
\addr    School of Artificial Intelligence,\\
Jilin University, Changchun, 130012, China \\        
\name Ivor W. Tsang  \email  ivor\_tsang@ihpc.a-star.edu.sg \\
\addr Centre for Frontier AI Research (CFAIR),\\
and \\
Institute of High-Performance Computing (IHPC), \\
Agency for Science, Technology and Research (A*STAR), Singapore \\
\name  James T.~Kwok \email jamesk@cse.ust.hk\\
\addr  Department of Computer Science and Engineering\\
The Hong Kong University of Science and Technology (HKUST)
}
\editor{xxxx}
 
\maketitle

\begin{abstract}
Hypothesis-pruning maximizes the hypothesis updates for active learning to find those desired unlabeled data. An inherent assumption is that this learning manner can derive those updates into the optimal hypothesis. However, its  convergence may not be guaranteed well  if those  incremental updates are  negative and disordered. In this paper, we introduce a   black-box teaching hypothesis $h^\mathcal{T}$    
employing a tighter slack term  $\left(1+\mathcal{F}^{\mathcal{T}}(\widehat{h}_t)\right)\Delta_t$ to replace the typical $2\Delta_t$ for pruning.  Theoretically, we prove that, under the guidance of this teaching hypothesis, the learner can converge into a tighter generalization error and label complexity bound than those non-educated learners who do not receive any guidance from a teacher:1)   the generalization error upper bound can be reduced from $R(h^*)+4\Delta_{T-1}$ to approximately $R(h^{\mathcal{T}})+2\Delta_{T-1}$, and 2)  the label complexity upper bound can be decreased from $4 \theta\left(TR(h^{*})+2O(\sqrt{T})\right)$ to approximately $2\theta\left(2TR(h^{\mathcal{T}})+3 O(\sqrt{T})\right)$. To be strict with our assumption, self-improvement of teaching is firstly proposed when $h^\mathcal{T}$  loosely approximates  $h^*$.
Against learning, we further consider two teaching scenarios: teaching a white-box and black-box learner. Experiments verify this idea  and show better generalization performance
 than the fundamental active learning strategies, such as IWAL \citep{beygelzimer2009importance}, IWAL-D \citep{cortes2019active}, etc.

\end{abstract}
\begin{keywords}
  Hypothesis Pruning, Black-box Teaching,  Active Learning, Error Disagreement,    Label  Complexity. 
\end{keywords}

\section{Introduction}
Hypothesis-pruning \citep{kaariainen2004selective}  interactively prunes a pre-specified hypothesis class $\mathcal{H}$ to find one desired output, which improves the convergence of any learning algorithm using as few labels as possible, such as active learning \citep{settles2009active}. The setting is that the learner has access to a pool of unlabeled data and can query labels from human annotators for those unlabeled data, where the hypotheses are generated from a functional assumption, e.g., MLP, CNN, etc; if the hypotheses do not rely on any functional assumption, it becomes an agnostic scenario \citep{balcan2009agnostic} that studies the theoretical performance of achieving a parameterized error by controlling the label complexity bound \citep{hanneke2007bound}. In theory aspect,  a set of hypothesis update methods and bound analyses of label complexity  were presented, e.g., \citep{hanneke2012activized} \citep{beygelzimer2009importance}; in practical applications,  active learning  already benefited the image annotation \citep{beluch2018power},   semantic segmentation \citep{siddiqui2020viewal}, etc.

There is one common assumption in active learning, whether in its theoretical explorations or practical applications: the infinite hypothesis class exists the optimal hypothesis that may be incrementally updated from one passive initialization. With this assumption, Hanneke \emph{et al.} proposed an error disagreement  coefficient \citep{hanneke2014theory} to control the hypothesis updates. The policy is that any disagreement generated from the candidate hypothesis larger than the pre-defined coefficient are feasible and positive updates \citep{cao2021distribution}. Otherwise, it is an insignificant update. To minimize the label complexity   of the updating costs, Zhang \emph{et al.} presented a tighter bound  using a new term called confidence rate \citep{zhang2014beyond}. To handle noisy samples, Golovin \emph{et al.} \citep{golovin2010near} proposed  a near-optimal Bayesian policy that invokes adaptive submodularity, generalizing  submodular set functions to adaptive policies. Under homogeneous halfspace learning,  Yan \emph{et al.} \citep{yan2017revisiting} presented  the near-optimal label complexity bounds  for the bounded and  adversarial noise conditions. Due to the unknown of the final convergence of  hypothesis, each round of the hypothesis update of active learning may be   negative, which may lead the learner to be disordered when updating towards a subsequent hypothesis. Therefore, the existing theoretical results may not guarantee well the  convergence of those incremental updates in hypothesis class, that is, the optimal hypothesis $h^*$ may not be easily obtained from these updates without explicit guidance and information from $h^*$.


In this paper, we introduce a black-box teacher \citep{dasgupta2019teaching, liu2018towards} who can provide guidance for the learner but does not disclose any its cue, such as the parameter distribution and convergence condition of the learner, etc. In this way, the teacher gives a black-box hypothesis $h^\mathcal{T}$, which maintains a fair teaching scenario compared to those non-educated learners who do not receive any guidance from a teacher. With $h^\mathcal{T}$, an active learner can easily  replace the infeasible $h^*$ and select those unlabeled data which  maximize the disagreement of the feedback between teacher and learner, not maximizing the disagreement of the current and subsequent hypotheses as typical active learning. Our contributions are summarized as follows.
\begin{itemize}
\item We propose a new perspective of introducing black-box machine teaching to guide an active learner, which guarantees a desired convergence to an approximated teaching hypothesis, not the typical infeasible optimal hypothesis. 


\item We theoretically prove that, under the guidance of the teaching hypothesis, the learner can converge into  tighter generalization error and label complexity bounds than those non-educated learners without teacher guidance. To further improve its generalization, we then consider two  scenarios: teaching a white-box and black-box learner, where the self-improvement of teaching is firstly proposed to improve the initial teaching hypothesis. 

\item We present a black-box teaching-based active learning (BTAL) algorithm, which spends fewer annotations to converge, yielding more effective performance than those typical active learning strategies. 
\end{itemize}

\textbf{Organization.} Section~\ref{sec:Related Work} presents the related work. Section~\ref{sec:Error Disagreement-based Active Learning} elaborates the error disagreement-based active learning. Section~\ref{sec:Black-box Teaching} explains our black-box teaching idea. Section~\ref{sec:Black-box Teaching-based Active Learning} employs this idea to guide an active learner. Experiments are presented in Section~\ref{sec:Experiments}. We conclude this work in Section~\ref{sec:Conclusion}. 

\textbf{Notation} We introduce the set of notations  used throughout the paper. We denote by $\mathcal{X}$ the input dataset and by $\mathcal{Y}$ the output label set. Let $\mathcal{D}$ be an unkonwn distribution over $\mathcal{X}\times\mathcal{Y}$, and $\mathcal{D}_{\mathcal{X}}$ be the marginal distribution of $\mathcal{D}$ over $\mathcal{X}$. We consider the on-line active learning scenario: for each time $t \in [T]=\{1,...,T\}$, the learner receives an input sample $x_t$ drawn i.i.d.  according to $\mathcal{D}_{\mathcal{X}}$ and has to decide whether to query its label.

We denote by $\mathcal{H}=\{h:\mathcal{X}\to\mathcal{Z}\}$ the hypothesis class, where $\mathcal{Z}$ is a prediction space. Let $\ell(h(x),y)$ denotes the loss function which operates $\mathcal{Z}\times\mathcal{Y}\to [0,1]$.  For any hypothesis $h$, we denote $R(h)$ to be the generalization error: $R(h)=\mathop{\mathbb{E} }\limits_{(x,y) \sim  \mathcal{D}}\left[\ell(h(x),y)\right]$, and denote  $h^{*}=\underset{h \in \mathcal{H}}{\operatorname{argmin}}\ R(h)$ to be the optimal hypothesis in $\mathcal{H}$. We also denote by $L_t(h)$ the importance-weighted empirical error of $h$,  defined by the weighted loss of query samples w.r.t. Eq.~(\ref{equ:importance-weighted empirical error}). Let $H_t$ denote the candidate hypothesis set of the learner at $t$-time, where $H_1=\mathcal{H}$. At $t$-time, we define the current empirical optimal  hypothesis $\widehat{h}_t=\underset{h \in H_t}{\operatorname{argmin}}\ L_{t}(h)$, which has the minimum importance-weighted empirical error in $H_t$. 

We use $h^{\mathcal{T}}$ to denote the teaching hypothesis w.r.t. Assumption~\ref{ass:T-AL}, which replaces the infeasible $h^*$ to guide the active learner.  Then, we use $\mathcal{H}^{\mathcal{T}}$ to denote the teaching-hypothesis-class w.r.t. Definition~\ref{def:teacher}, which is an efficient approximation to $\mathcal{H}$. To avoid any confusion, we denote by $H^{\mathcal{T}}_t$   the candidate hypothesis set at $t$-time with respect to the teaching hypothesis $h^{\mathcal{T}}$. With a slight abuse of notation, we also use $\widehat{h}_t=\underset{h \in H^{\mathcal{T}}_t}{\operatorname{argmin}}\ L_{t}(h)$ to denote  the current empirical optimal hypothesis at $t$-time in $H^{\mathcal{T}}_t$.

\section{Main Theoretical Results}
\textbf{Main Progress} Theoretically, we present the teaching-based hypothesis pruning and its improvements to defend our teaching idea. In detail, 1) we observe whether the teaching-based hypothesis pruning strategy can prune the candidate hypothesis set faster than the error disagreement-based active learning; 2) we observe whether the optimal hypothesis can be usually maintained in the candidate hypothesis set; 3) we also present the generalization error and label complexity bounds of teaching an active learner.

\textbf{Main Assumption}
For any hypothesis class $\mathcal{H}$, assume that there exists a teaching hypothesis  $h^{\mathcal{T}}$ which  tolerates  an error bias $\epsilon$:
\begin{equation*}
\begin{split}
\mathcal{L}\left(h^*,h^{\mathcal{T}}\right)&=\mathop{\mathbb{E} }\limits_{x \sim  \mathcal{D}_{\mathcal{X}}}\left[\max_{y}\big|\ell(h^*(x),y)- \ell(h^{\mathcal{T}}(x),y)\big|\right] < \epsilon,
\end{split}
\end{equation*}
where $h^*$ is the optimal hypothesis in $\mathcal{H}$, and the disagreement of hypothesis invokes Eq.~(\ref{equ:disagreement}).

 \textbf{Main Technique}
 We still follow the pruning manner of IWAL w.r.t. Eq.~(\ref{equ:IWAL_hpy_pruning}) to supervise the updates of the candidate hypothesis set, where the main difference is that we introduce a teaching hypothesis $h^{\mathcal{T}}$ to control the slack  constraint of hypothesis pruning. Specifically, the slack  constraint $2\Delta_t$ is tightened as $\left(1+\mathcal{F}^{\mathcal{T}}(\widehat{h}_t)\right)\Delta_t$ by invoking the guidance of a teacher, where $\mathcal{F}^{\mathcal{T}}(\widehat{h}_t)$ denotes disagreement feedback with the teacher w.r.t. current empirical optimal hypothesis $\widehat{h}_t$. With such operation, the candidate hypothesis set  $ H^{\mathcal{T}}_{t+1}$ at $t+1$-time is updated by  
\begin{equation*}\label{equ:hyp-pruning}
H^{\mathcal{T}}_{t+1}\!=\! \left\{h\in H^{\mathcal{T}}_{t}:L_t(h)\leq L_t(\widehat{h}_t)+\left(1+\mathcal{F}^{\mathcal{T}}(\widehat{h}_t)\right)\Delta_t\right\},
\end{equation*}
where $H^{\mathcal{T}}_{1}=\mathcal{H}^{\mathcal{T}}$, and $\Delta_{t}=\sqrt{(2/t)\log(2t(t+1)|\mathcal{H}^{\mathcal{T}}|^{2}/ \delta)}$ for some fixed confidence parameter $\delta>0$. Therefore teaching-based hypothesis pruning is more aggressive in shrinking the candidate hypothesis set, resulting in better learning guarantees.

\textbf{Main Theorem 0.1}\emph{ For any teaching-hypothesis-class $\mathcal{H}^{\mathcal{T}}$, teaching an active learner runs on $\mathcal{H}^{\mathcal{T}}$. Given any $\delta > 0$, with a probability at least $1-\delta$, for any $T\in\mathbb{N}^{+}$, the following holds: }

1) \emph{the generalization error holds} \[R(\widehat{h}_T)\le R(h^{*})+\left(2+\mathcal{F}^{\mathcal{T}}(\widehat{h}_{T-1})+\mathcal{F}^{\mathcal{T}}(\widehat{h}_T)\right)\Delta_{T-1}+\epsilon;\]
2) \emph{if the learning problem has disagreement coefficient $\theta$, the label complexity is at most} \[
\tau_T \leq 2 \theta\left(2TR(h^{*})+\big(3+\mathcal{F}^{\mathcal{T}}(\widehat{h}_{T-1}))\big) O(\sqrt{T})+2T\epsilon\right).\]

The above Theorem shows the generalization error and label complexity bounds of black-box teaching an active learner. The performance of black-box teaching relies on two key factors: firstly, the effectiveness of active learning, i.e., the magnitude of the teacher's disagreement feedback $\mathcal{F}(\widehat{h})$ of the learner; and secondly, the quality of the teacher, which is determined by the maximum level of disagreement $\epsilon$ between the teaching hypothesis and the optimal hypothesis. In particular, when $\widehat{h}_{T-1},\widehat{h}_{T}$ are sufficiently close to $h^{\mathcal{T}}$ and $\epsilon$ is a tolerable error, the generalization error upper bound can be reduced from $R(h^*)+4\Delta_{T-1}$ to approximately $R(h^{\mathcal{T}})+2\Delta_{T-1}$, and the label complexity upper bound can be decreased from $4 \theta\left(TR(h^{*})+2O(\sqrt{T})\right)$ to approximately $2\theta\left(2TR(h^{\mathcal{T}})+3 O(\sqrt{T})\right)$.

  \textbf{More strict assumption :} \emph{If the teaching hypothesis is loosely approximated to the optimal hypothesis, i.e. $\epsilon$ is large, how do we guarantee the convergence of black-box teaching? We thus design self-improvement of teaching. }

\textbf{Main Theorem 0.2} 
\emph{For any teaching-hypothesis-class $\mathcal{H}^{\mathcal{T}}$, teaching an active learner runs on $\mathcal{H}^{\mathcal{T}}$. If the self-improvement of teaching is applied, given any $\delta > 0$, with a probability at least $1-\delta$, for any $T\in\mathbb{N}^{+}$, the following holds: 1) for any $t\in[T]$, holds $h^{\mathcal{T}}_t\in H^{\mathcal{T}}_t$; }

2) \emph{the generalization error holds}
\begin{equation*}
R(\widehat{h}_T) \leq R(h^{*})+\left(2+\mathcal{F}^{\mathcal{T}}_{T-1}(\widehat{h}_{T-1})+\mathcal{F}^{\mathcal{T}}_{T-1}(\widehat{h}_T)\right)\Delta_{T-1} + \epsilon_{T-1};
\end{equation*}
3) \emph{if the learning problem has disagreement coefficient $\theta$, the label complexity is at most}
\[
\tau_T  \leq 2 \theta\left(2TR(h^{*})+\big(3+\mathcal{F}^{\mathcal{T}}_{T-1}(\widehat{h}_{T-1})\big) O(\sqrt{T})+2T\epsilon_{T-1}\right).\]
 The Theorem  shows that the optimal hypothesis of $\bigcup\limits_{k=1}^t H^{\mathcal{T}}_k$ is maintained in the candidate hypothesis set with a high probability at any $t$-time. Recalling Corollary~\ref{cor:self-teacher}, there exists $\epsilon = \epsilon_1 \leq \epsilon_{T-1}$, which shows that self-improvement of teaching strategy can further reduce the generalization error and label complexity bounds of the learner w.r.t. Theorem~\ref{thm:TAL-learning guarantees}. Moreover, the improvement of the active learner is decided by the improvement of the black-box teacher.

Concretely, by generating new hypotheses, self-improvement of teaching strategy tightens the approximation of the teaching hypothesis to the optimal hypothesis, which provides more favorable learning guarantees for an active learner.
 
\section{Related Work}\label{sec:Related Work}
We firstly introduce the active learning including its theoretical explorations and practical applications. We then present the machine teaching which supervises a white-box and black-box  learner.
\subsection{Active Learning}
Active learning has two branches: theoretical explorations \citep{hanneke2009theoretical} and practical applications \citep{settles2009active}, where the theoretical scenario focuses on the generalization analysis on error and label complexity bounds of hypothesis class, and the practical applications generalize  those theoretical results into   weakly-supervised sampling \citep{rasmus2015semi}, Bayesian  approximation \citep{pinsler2019bayesian},  adversarial training \citep{sinha2019variational},
etc.

\textbf{Theoretical explorations} There are two perspectives on theoretical active learning: agnostic bound convergence and version space shrinking, where agnostic active learning is derived from the standard  PAC framework \citep{denis1998pac}, and version space shrinking \citep{dasgupta2004analysis,tong2001support} can be generalized from a hypothesis pruning view \citep{cortes2019active,cao2020shattering}. Under linear perceptron analysis \citep{gonen2013efficient}, Dasgupta \emph{et al.} \citep{dasgupta2011two} presented a series of upper and lower bounds on label complexity, keeping consistent convergence as query-by-committee  \citep{gilad2006query}  algorithm which employs multiple learners. Hannker then extended their bounds for more general settings, such as \citep{hanneke2007bound, hanneke2012activized} and enhanced efficiency of the error disagreement coefficient. For a uniform framework, MF Balcan   \emph{et al.} summarized those theoretical results as agnostic scenario \citep{balcan2009agnostic}. However, those results always assume a uniform distribution and noise-free setting. For bounded and  adversarial noise, Yan  \emph{et al.}   \citep{yan2017revisiting}  presented the label complexity bounds. With consistent assumption of support vectors, Tong \emph{et al.} \citep{tong2001support} use  the notion of  version space to  shrink its volume  by maximizing  the minimum
distance to any of the delineating hyperplanes. The other similar works can also be found in   \citep{warmuth2001active,golovin2010adaptive,ailon2012active,krishnamurthy2017active}. To shirk the version space into the minimal covering on the optimal hypothesis,
Cortes \emph{et al.}   presented  a region-splitting algorithm to guarantee that
the pruning in the hypothesis class can converge into the optimal hypothesis, e.g., \citep{cortes2019region, cortes2020adaptive}.

\textbf{Practical applications} To pure hypothesis class, following the error disagreement coefficient, incremental optimization is a typical way, that is, iteratively update the current learning model by maximizing its uncertainty. With this paradigm,  various of baselines were proposed, such as the maximize error reduction \citep{roy2001toward},   maximize mean standard deviation \citep{kampffmeyer2016semantic}, etc. In statistical optimization, active learning also can be redefined as experimental design \citep{wong1994comparing} including 
the A, D, E, and T optimal design,   where the A-optimal design  minimizes the average variance of the parameter estimates, 
D-optimal design maximize the differential Shannon information content of the parameter estimates,
 E-optimal design maximizes the minimum eigenvalue of the information matrix, and T-optimal design  methods maximizes the trace of the information matrix. 
 On the Bayesian setting,  active learning is defined as the Bayesian approximation on likelihood \citep{orekondy2019knockoff} or maximize  the information gain \citep{kirsch2019batchbald}, etc.
 In recent years, benefiting from the powerful modeling of deep neural networks, deep active learning was proposed and brought new interests, e.g., Monte-Carlo dropout with active learning \citep{gal2017deep}, deep active annotation \citep{huijser2017active},  adversarial training  with active querying set \citep{sinha2019variational}, dual adversarial network for deep active learning \citep{wang2020dual}, consistency-based semi-supervised active learning \citep{gao2020consistency}, etc.

\subsection{Machine Teaching}
Machine teaching \citep{zhu2018overview} studies an inverse problem of machine learning, that is, finding the optimal teaching examples if the teacher already knows the learning parameters. There are two scenarios for machine teaching: white-box teaching and black-box teaching \citep{dasgupta2019teaching, liu2018towards}.

\textbf{White-box teaching} Machine teaching assumes that the teacher knows the optimal learning parameter of the   learner. It has provided theoretical analyses for    linear regression learner,  logistic regression learner, and SVM learner, etc., to find their best teaching examples, where those examples may update a random initial training parameter into their optimal. In other words, it provides an optimal control on the parameter exploration for a  learner.  To  improve the theoretical guarantees, Goldman \emph{et al.}  \citep{goldman1995complexity} presented a complete of theoretical concepts including teaching dimension \citep{liu2016teaching,doliwa2014recursive}, teaching complexity  \citep{hanneke2007teaching}, etc.   Zhu  \emph{et al.} \citep{zhu2017no} then presented the teaching theories for multiple learners.
In practical scenarios, this teaching style has been widely used in teacher-student learning model, e.g., \citep{wang2020progressive,matiisen2019teacher,meng2018adversarial,meng2019domain}. 

\textbf{Black-box teaching} There is one more challenging problem that  the teacher may not disclose any cue of the distribution of the learning parameters, that is, the learner may be a black-box.  In such scenario,  Liu \emph{et al.}  \citep{liu2018towards}  considered the cross-space machine teaching, which invokes different feature representations for teacher and student.  Dasgupta \emph{et al.} \citep{dasgupta2019teaching} proposed to shrink the training sets for any family of classifiers by finding an approximately-minimal subset of training instances that yields consistent properties as the original hypothesis class. Cicalese \emph{et al.}\citep{cicalese2020teaching} consider the case where the teacher can only aim at having the learner converge to a well-available approximation of the optimal hypothesis. Cao \emph{et al.} \citep{cao2021_distribution}  proposed to use iterative distribution matching to teach a black-box learner.  Orekondy \emph{et al.} \citep{orekondy2019knockoff}  steal the model functionality  feedback to guide a black-box  learner, that is, approximating their parameter distributions.

\section{Error Disagreement-based Active Learning}\label{sec:Error Disagreement-based Active Learning}
In this section, we introduce the error disagreement and its generalized learning algorithm with guarantees. We consider the error disagreement-based active learning because it can be applied with different classifiers and can be studied under a variety of noise models. See \citet{hanneke2014theory} for more introduced and established results of the error disagreement-based active learning.

\subsection{Error Disagreement}
Given a hypothesis class $\mathcal{H}$, active learning tries to reduce the maximum disagreement of hypothesis in $\mathcal{H}$ by invoking a disagreement function $\mathcal{L}(\cdot,\cdot)$ \citep{cortes2019active}. 
   


For any hypothesis pair  $\{ h, h'\} \subseteq \mathcal{H}$, $\mathcal{L}(h, h')$ measures their disagreement by the error disagreements, i.e.,
\begin{equation}\label{equ:disagreement}
\mathcal{L}(h,h')= \mathop{\mathbb{E} }\limits_{x \sim  \mathcal{D}_{\mathcal{X}}}\left[\max_{y} \big|\ell(h(x),y)-\ell(h'(x),y)\big|\right],
\end{equation}
where $\ell(h(x),y)$ denotes the loss function which operates $\mathcal{X}\times\mathcal{Y}\to [0,1]$. The calculation of error disagreement w.r.t. Eq.~(\ref{equ:disagreement}) does not require labels, i.e., it can be calculated over the unlabeled dataset. Given an i.i.d. sample $x_1,x_2,\dots,x_n$ from $\mathcal{D}_{\mathcal{X}}$, the error disagreement is the empirical average $\mathcal{L}(h,h')=\frac{1}{n}\sum_{i=1}^n\left[\max_{y} \big|\ell(h(x_i),y)-\ell(h'(x_i),y)\big|\right]$. As an example of binary classification, we solve for $\max\left\{\big|\ell(h(x_i),+1)-\ell(h'(x_i),+1)\big|,\big|\ell(h(x_i),-1)-\ell(h'(x_i),-1)\big|\right\}$ to obtain the error disagreement between $h$ and $h'$ over sample $x_i$.


\subsection{Learning Algorithm}\label{subsec:Learning Algorithm}

Importance weighted active learning (IWAL) \citep{beygelzimer2009importance} invokes the error disagreement to prune the hypothesis class $\mathcal{H}$, which is a typical error disagreement-based active learning algorithm. 


Given an initial candidate hypothesis set $H_1=\mathcal{H}$, IWAL receives $x_t\in\mathcal{X}$ drawn i.i.d. according to $\mathcal{D}_{\mathcal{X}}$. At $t$-time, the algorithm decides whether to query the label of $x_t$ and prunes the candidate hypothesis set $H_t$ to $H_{t+1}$.

\textbf{Query} At $t$-time, IWAL does a Bernoulli trial $Q_t$ with success probability $p_t$, where $p_t$ is the maximum error disagreement of $H_t$ over $x_t$:
\begin{equation}\label{equ:IWAL_pt}
p_t=\max_{h,h'\in H_t}\max_{y} \big|\ell(h(x_t),y)- \ell(h'(x_t),y)\big|.
\end{equation}
If $Q_t=1$, the algorithm queries the label $y_t$ of $x_t$. 

\textbf{Hypothesis pruning} Let $L_t(h)$ be the importance-weighted empirical error of hypothesis $h \in \mathcal{H}$, there exists:
\begin{equation}\label{equ:importance-weighted empirical error}
L_{t}(h)=\sum_{k=1}^{t} \frac{Q_{k}}{p_{k}} \ell\left(h\left(x_{k}\right), y_{k}\right),
\end{equation}
where its minimizer is $\widehat{h}_t=\underset{h \in H_t}{\operatorname{argmin}}\ L_{t}(h)$. With the expectation taken over all the random variables, we know $\mathbb{E}\left[L_t(h)\right]=R(h)$.  At $t$-time, IWAL prunes  $H_t$ to $H_{t+1}$ through $L_t(\widehat{h}_t)$ and an allowed slack $2\Delta_t$:
\begin{equation}\label{equ:IWAL_hpy_pruning}
H_{t+1}=\left\{h\in H_{t}:L_t(h)\leq L_t(\widehat{h}_t)+2\Delta_t\right\},
\end{equation}
where $\Delta_{t}=\sqrt{(2/t)\log(2t(t+1)|\mathcal{H}|^{2}/ \delta)}$ for a fixed confidence parameter $\delta>0$. At $T$-time, IWAL returns the current empirical optimal hypothesis $\widehat{h}_{T}$ as the final hypothesis output.

We add some remarks on evaluating the quality of active learning algorithms. The following Remark~\ref{re:success} presents the necessary conditions for a feasible active learning algorithm.
\begin{remark}\label{re:success}
Whether the optimal hypothesis $h^*$ can usually be maintained in the candidate hypothesis set $H_t$ is a necessary condition for the success of an active learning algorithm.
\end{remark}

The following Remark~\ref{re:quality} presents two factors for evaluating the quality of an active learning algorithm.
\begin{remark}\label{re:quality}
Two factors measure the quality of an active learning algorithm: 1)  tighter bound on generalization error $R(\widehat{h}_T)$, where $\widehat{h}_T$ is the hypothesis returned by the algorithm after $T$ rounds, and 2) tighter bound on label complexity $\tau_T$, where $\tau_T$ is the expected value of label numbers queried by the active learning algorithm within $T$ rounds.
\end{remark}
With Remarks~\ref{re:success} and \ref{re:quality}, to guarantee a high-quality learning performance, any active learning algorithm needs to satisfy the three factors, including 1) maintaining the optimal hypothesis, 2) tighter bound on generalization error, and 3) tighter bound on label complexity.

\subsection{Learning Guarantees}
We present the learning guarantees analysis for IWAL. Firstly, we introduce another definition of the disagreement with respect to hypothesis. For any two hypotheses $h,h'$, let $\rho(h,h')$ denote their disagreement:
\begin{equation}\label{equ:new disagreement}
\rho(h,h')=\mathop{\mathbb{E} }\limits_{(x,y) \sim  \mathcal{D}}\left[\big|\ell(h(x),y)-\ell(h'(x),y)\big|\right].
\end{equation}
The new disagreement $\rho(\cdot,\cdot)$ can derive a more favorable learning guarantees for the error disagreement-based active learning. \citet{cortes2019active} shows that the new disagreement $\rho(\cdot,\cdot)$ removes a constant $K_{\ell}$ from the label complexity bound of IWAL compared to the error disagreement $\mathcal{L}(\cdot,\cdot)$ w.r.t. Eq.~(\ref{equ:disagreement}). Based on the new disagreement $\rho(\cdot,\cdot)$, we can define a ball with respect to the hypothesis. Given $r>0$, let $B(h^*,r)$ denote a ball centered in $h^*\in \mathcal{H}$ with the radius $r$: $B(h^*,r)=\left\{h \in \mathcal{H}: \rho(h^*,h)\leq r \right\}$, where $h^*$ is the optimal hypothesis of  $\mathcal{H}$. The error  disagreement coefficient is then defined as the minimum value of $\theta$ for all $r>0$:
\begin{equation}\label{equ:error disagreement coefficient}
\theta \geq  \mathop{\mathbb{E} }\limits_{x \sim  \mathcal{D}_{\mathcal{X}}} \left[ \max_{h \in B(h^*, r)} \max_y \frac{\big|\ell(h(x),y)-\ell(h^*(x),y)\big|}{r} \right].
\end{equation}

The error disagreement coefficient $\theta$ is a complexity measure widely used for label complexity analysis in disagreement-based active learning. See \citet{hanneke2014theory} for more analysis of disagreement coefficient in active learning. Based on the error disagreement coefficient,   guarantees of the learning algorithm is proved by \citet{beygelzimer2009importance} and improved by \citet{cortes2019active}.


\begin{theorem}\label{thm:IWAL}
For any hypothesis class $\mathcal{H}$, {\rm IWAL} runs on $\mathcal{H}$. Given any $\delta > 0$, with probability at least $1-\delta$, for any $T\in\mathbb{N}^{+}$, the following holds: 1) for any $t\in[T]$, holds $h^{*}\in H_t$; 2) the generalization error holds $R(\widehat{h}_T)\le R(h^*)+4\Delta_{T-1}$; 3) if the learning problem has a disagreement coefficient $\theta$, the label complexity is at most $\tau_T \leq 4 \theta\left(TR(h^{*})+2O(\sqrt{T})\right).$
\end{theorem}

Theorem~\ref{thm:IWAL} guarantees the following facts. 1) The optimal hypothesis $h^{*}$ is maintained in the candidate hypothesis set $H_t$ with high probability, which is the key to the success of IWAL. 2) As time $T$ increases, $\Delta_{T-1}$ gradually tends to zero, leading to a tighter approximation of $\widehat{h}_T$ to $h^{*}$ in terms of $R(\widehat{h}_T)-R(h^*)$. 3) The upper bound on the number of query labels of IWAL depends on the disagreement coefficient $\theta$.

\section{Black-box Teaching}\label{sec:Black-box Teaching}

Error disagreement-based active learning may not easily prune the candidate hypotheses into their optimum. We thus introduce a teaching hypothesis that guides an active learner to converge with tighter bounds on generalization error and label complexity.

Theoretical analysis   prove  that 1) we  introduce a teaching hypothesis to guide the hypothesis pruning, which results in faster  pruning speed but  always  retains the optimal hypothesis in the candidate hypothesis set; 2) to improve the initial teaching hypothesis, self-improvement is applied and shows better learning guarantee than any initialization on the teacher. Related proofs are presented in Appendix~\ref{app_proof}.

\subsection{Teaching Assumption}\label{subsec:Teaching Assumption}
The primary assumption of our black-box teaching idea is formed as follows.
\begin{assumption}\label{ass:T-AL}
For any hypothesis class $\mathcal{H}$, assume that there exists a teaching hypothesis  $h^{\mathcal{T}}$ which  tolerates  an error bias $\epsilon$:
\begin{equation*}
\begin{split}
\mathcal{L}\left(h^*,h^{\mathcal{T}}\right)&=\mathop{\mathbb{E} }\limits_{x \sim  \mathcal{D}_{\mathcal{X}}}\left[\max_{y}\big|\ell(h^*(x),y)- \ell(h^{\mathcal{T}}(x),y)\big|\right] < \epsilon,
\end{split}
\end{equation*}
where $h^*$ is the optimal hypothesis in $\mathcal{H}$, and the disagreement of hypothesis invokes Eq.~(\ref{equ:disagreement}).
\end{assumption}
Note that Assumption~\ref{ass:T-AL} presents a formal description for our teaching idea, and we also consider a loose approximation of $h^{\mathcal{T}}$ in Section~\ref{subsec:Self-improvement of Teaching}, i.e., $\epsilon$ is large. In real-world scenarios, it is a more practical problem and can help to improve  the credibility of our assumption.

With Assumption~\ref{ass:T-AL}, we then  construct an approximation to the hypothesis class  $\mathcal{H}$.
\begin{definition}\label{def:teacher}\textbf{Teaching-hypothesis-class.}
For any hypothesis class $\mathcal{H}$, $h^{\mathcal{T}}$ is a teaching hypothesis that satisfies Assumption~\ref{ass:T-AL}. If there exists a hypothesis class $\mathcal{H}^{\mathcal{T}}$ s.t. $h^{\mathcal{T}}=\underset{h \in \mathcal{H}^{\mathcal{T}}}{\operatorname{argmin}}\ R(h)$, then $\mathcal{H}^{\mathcal{T}}$ is called the teaching-hypothesis-class of $\mathcal{H}$.
\end{definition}

By introducing a black-box teaching hypothesis $h^{\mathcal{T}}$, we define a new hypothesis class $\mathcal{H}^{\mathcal{T}}$ related to $\mathcal{H}$, which uses $h^{\mathcal{T}}$ to replace the infeasible $h^*$. The following two feasible corollaries show the validity of Definition~\ref{def:teacher}.


\begin{corollary}\label{cor:complexity}
For any hypothesis class $\mathcal{H}$ and given $h^{\mathcal{T}}$,  
$\mathcal{H}^{\mathcal{T}}$ is a teaching-hypothesis-class of $\mathcal{H}$. Then there exists the inequality $\mathop{\mathbb{E}}\limits_{(x,y) \sim  \mathcal{D}}\left[ \ell(h^*(x),y)- \ell(h^{\mathcal{T}}(x),y) \right]\leq 0$, which requires that there are at least   $\tau\geq 0$ hypotheses with tighter generalization errors than $h^{\mathcal{T}}$. Therefore, the teaching-hypothesis-class $\mathcal{H}^\mathcal{T}$ has fewer candidate hypotheses  than $\mathcal{H}$, that is, $|\mathcal{H}^\mathcal{T}|\leq|\mathcal{H}|$.
\end{corollary} 
 
For any learning algorithm, Corollary~\ref{cor:complexity} shows that hypothesis pruning in $\mathcal{H}^{\mathcal{T}}$ may have lower complexity than that of $\mathcal{H}$. Corollary~\ref{cor:complexity} gives the validity of $\mathcal{H}^{\mathcal{T}}$ in terms of complexity, and the following corollary gives the validity of $\mathcal{H}^{\mathcal{T}}$  in terms of error. 


\begin{corollary}\label{cor:error} For any hypothesis class $\mathcal{H}$ and given $h^{\mathcal{T}}$, $\mathcal{H}^{\mathcal{T}}$ is a teaching-hypothesis-class of $\mathcal{H}$. Based on properties of expectation, we have $\big|R(h^{\mathcal{T}})-R(h^*)\big|\leq\mathcal{L}(h^*,h^{\mathcal{T}})<\epsilon$. 


\end{corollary}
For any learning algorithm, Corollary~\ref{cor:error} shows that the error of hypothesis pruning in $\mathcal{H}^{\mathcal{T}}$ is almost equal to the error of hypothesis pruning in $\mathcal{H}$. In conclusion, Corollary~\ref{cor:complexity}-\ref{cor:error} initially demonstrates the validity of our teaching idea. The subsequent theorems in this paper strictly give the improved bounds on generalization error and label complexity.




\subsection{Teaching Model}
Before precisely presenting our theoretical results, we set some notes and explain the black-box teaching model in more detail. We use $\mathcal{F}^{\mathcal{T}}(\cdot)= \mathcal{L}(h^{\mathcal{T}},\cdot)$ to denote a disagreement feedback function with operation $\mathcal{H}\to[0,1]$.

\textbf{Teacher}: the teacher has a teaching hypothesis $h^{\mathcal{T}}$, which only can provide the disagreement feedback $\mathcal{F}^{\mathcal{T}}(\cdot)$ to the \textbf{Learner}.

\textbf{Learner}: the learner has a teaching-hypothesis-class $\mathcal{H}^{\mathcal{T}}$, which prunes $\mathcal{H}^{\mathcal{T}}$ by identifying the disagreement feedback $\mathcal{F}^{\mathcal{T}}(\cdot)$ with the \textbf{Teacher}.

At $t$-time, the learner receives a sample $x_t$ and decides whether to query the label $y_t$ of $x_t$. Then the learner prunes the candidate hypothesis set based on the disagreement feedback $\mathcal{F}^{\mathcal{T}}(\cdot)$ with the teacher. The goal of the learner is to return a desired hypothesis $\widehat{h}$ from $\mathcal{H}^{\mathcal{T}}$ by using fewer labeled samples, where $\widehat{h}$ has the minimum generalization error on the  input dataset $\mathcal{X}$.

The black-box teaching scenario we consider is simple and practical, which merely necessitates the teacher’s ability to provide the learner with disagreement feedback. In this setting, the teacher is required to be  an end-to-end  model which only provides output as the feedback of the input and does not know the model configuration. Therefore, the learner only requires very limited information from the teacher, which maintains a fair teaching scenario compared to those non-educated learners who do not receive any guidance from a teacher.

We follow the rules of notations used in a standard hypothesis pruning  like IWAL of Section~\ref{sec:Error Disagreement-based Active Learning}. Let $H^{\mathcal{T}}_t$ denote the candidate hypothesis set of the learner at $t$-time, where $H^{\mathcal{T}}_1 = \mathcal{H}^{\mathcal{T}}$. We denote by $\widehat{h}_t=\underset{h \in H^{\mathcal{T}}_t}{\operatorname{argmin}}\ L_{t}(h)$ the current empirical optimal  hypothesis, which has the minimum importance-weighted empirical error in $H^{\mathcal{T}}_t$. At $T$-time, the algorithm returns the current empirical  optimal hypothesis $\widehat{h}_T$ as the final hypothesis output.


\subsection{Teaching Improves Hypothesis Pruning}\label{subsec:Teaching Improves hypothesis pruning}
We below present the teaching-based hypothesis pruning and its theoretical improvements to defend our teaching idea. In detail, 1) we observe whether the teaching-based hypothesis pruning strategy can prune the candidate hypothesis set faster than the error disagreement-based active learning; 2) we observe whether the optimal hypothesis can be usually maintained in the candidate hypothesis set; 3) we also present the generalization error and label complexity bounds of teaching an active learner.



\textbf{Teaching-based hypothesis pruning} We still follow the pruning manner of IWAL w.r.t. Eq.~(\ref{equ:IWAL_hpy_pruning}) to supervise the updates of the candidate hypothesis set, where the main difference is that we introduce a teaching hypothesis $h^{\mathcal{T}}$ to control the slack  constraint of hypothesis pruning. Specifically, the slack  constraint $2\Delta_t$ is tightened as $\left(1+\mathcal{F}^{\mathcal{T}}(\widehat{h}_t)\right)\Delta_t$ by invoking the guidance of a teacher, where $\mathcal{F}^{\mathcal{T}}(\widehat{h}_t)$ denotes disagreement feedback with the teacher w.r.t. current empirical optimal hypothesis $\widehat{h}_t$. With such operation, the candidate hypothesis set  $ H^{\mathcal{T}}_{t+1}$ at $t+1$-time is updated by  
\begin{equation}\label{equ:hyp-pruning}
H^{\mathcal{T}}_{t+1}\!=\! \left\{h\in H^{\mathcal{T}}_{t}:L_t(h)\leq L_t(\widehat{h}_t)+\left(1+\mathcal{F}^{\mathcal{T}}(\widehat{h}_t)\right)\Delta_t\right\},
\end{equation}
where $H^{\mathcal{T}}_{1}=\mathcal{H}^{\mathcal{T}}$, and $\Delta_{t}=\sqrt{(2/t)\log(2t(t+1)|\mathcal{H}^{\mathcal{T}}|^{2}/ \delta)}$ for some fixed confidence parameter $\delta>0$. Therefore teaching-based hypothesis pruning is more aggressive in shrinking the candidate hypothesis set, resulting in better learning guarantees.

\textbf{Pruning speed} With a fast hypothesis pruning speed, the candidate hypothesis set $H^{\mathcal{T}}_t$ is shrunk rapidly, which reduces the learning difficulty, easily converting into $h^\mathcal{T}$. The primary determinant of pruning speed is the pruning slack term, i.e.,  $\left(1+\mathcal{F}^{\mathcal{T}}(\widehat{h}_t)\right)\Delta_t$ of Eq.~(\ref{equ:hyp-pruning}). With Eqs.~(\ref{equ:IWAL_hpy_pruning}) and (\ref{equ:hyp-pruning}), there exists  $\left(1+\mathcal{F}^{\mathcal{T}}(\widehat{h}_t)\right)\Delta_t \leq 2\Delta_t$, which means that the teaching-based hypothesis pruning employs a tighter slack term  to shrink $H^{\mathcal{T}}_t$ than IWAL. It  then leads to a faster pruning speed for our teaching strategy. Therefore, our teaching-based hypothesis pruning may be easier to prune the candidate hypotheses into their optimum than the error disagreement-based hypothesis pruning.



\textbf{Retain the teaching hypothesis} To evaluate Remark~\ref{re:success} of teaching-based hypothesis pruning, we present our analysis. The following lemma relates importance-weighted empirical error to the generalization error.

\begin{lemma}{\label{lem:T-AL}}
For any teaching-hypothesis-class $\mathcal{H}^{\mathcal{T}}$, teaching an active learner runs on $\mathcal{H}^{\mathcal{T}}$, where the sequence of candidate hypothesis sets satisfies $H^{\mathcal{T}}_{t+1}\subseteq H^{\mathcal{T}}_{t}$ with $H^{\mathcal{T}}_1 = \mathcal{H}^{\mathcal{T}}$. Given any $\delta > 0$, with a probability at least $1-\delta$, for any $T\in\mathbb{N}^{+}$ and for all $h,h'\in H^{\mathcal{T}}_T$, the following inequality holds:
\begin{equation*}
\left|L_T(h)-L_T(h')-\left(R(h)-R(h')\right)\right|\leq\left(1+\mathcal{L}(h,h')\right)\Delta_T,
\end{equation*}
where $\Delta_{T}=\sqrt{(2/T)\log(2T(T+1)|\mathcal{H}^{\mathcal{T}}|^{2}/ \delta)}$.
\end{lemma}

Lemma~\ref{lem:T-AL} indicates that the generalization error is concentrated near its importance-weighted empirical error for every pair $\{h,h'\}\subseteq\mathcal{H}$. Based on Lemma~\ref{lem:T-AL}, we can derive the  Theorem~\ref{thm:retain the teacher}, which connects the importance-weighted empirical error of the teacher and the learner.
\begin{theorem}\label{thm:retain the teacher}
For any teaching-hypothesis-class $\mathcal{H}^{\mathcal{T}}$, teaching an active learner runs on $\mathcal{H}^{\mathcal{T}}$. Given any $\delta > 0$, with a probability at least $1-\delta$, for any $t\in\mathbb{N}^{+}$, the following inequality holds:
\begin{equation*}
L_t(h^{\mathcal{T}})-L_t(\widehat{h}_t)\leq\left(1+\mathcal{F}^{\mathcal{T}}(\widehat{h}_t)\right)\Delta_t.  
\end{equation*}
\end{theorem}
Theorem~\ref{thm:retain the teacher} shows that the teaching hypothesis $h^{\mathcal{T}}$  satisfies the pruning rule with a high probability at any $t$-time. And $h^{\mathcal{T}}$ is the optimal hypothesis in the teaching-hypothesis-class $\mathcal{H}^{\mathcal{T}}$. Thus teaching-based hypothesis pruning maintains the optimal hypothesis in the candidate hypothesis set with a high probability.



\textbf{Learning guarantees} To demonstrate the improvement of teaching an active learner (w.r.t. Remark~\ref{re:quality}), we present the learning guarantee for teaching-based hypothesis pruning in Theorem~\ref{thm:TAL-learning guarantees}.
\begin{theorem}\label{thm:TAL-learning guarantees}
For any teaching-hypothesis-class $\mathcal{H}^{\mathcal{T}}$, teaching an active learner runs on $\mathcal{H}^{\mathcal{T}}$. Given any $\delta > 0$, with a probability at least $1-\delta$, for any $T\in\mathbb{N}^{+}$, the following holds: 

1) the generalization error holds \[R(\widehat{h}_T)\le R(h^{*})+\left(2+\mathcal{F}^{\mathcal{T}}(\widehat{h}_{T-1})+\mathcal{F}^{\mathcal{T}}(\widehat{h}_T)\right)\Delta_{T-1}+\epsilon;\]
2) if the learning problem has disagreement coefficient $\theta$, the label complexity is at most \[
\tau_T \leq 2 \theta\left(2TR(h^{*})+\big(3+\mathcal{F}^{\mathcal{T}}(\widehat{h}_{T-1}))\big) O(\sqrt{T})+2T\epsilon\right).\]
\end{theorem}

Theorem~\ref{thm:TAL-learning guarantees} shows the generalization error and label complexity bounds of black-box teaching an active learner. The performance of black-box teaching relies on two key factors: firstly, the effectiveness of active learning, i.e., the magnitude of the teacher's disagreement feedback $\mathcal{F}(\widehat{h})$ of the learner; and secondly, the quality of the teacher, which is determined by the maximum level of disagreement $\epsilon$ between the teaching hypothesis and the optimal hypothesis. In particular, when $\widehat{h}_{T-1},\widehat{h}_{T}$ are sufficiently close to $h^{\mathcal{T}}$ and $\epsilon$ is a tolerable error, the generalization error upper bound can be reduced from $R(h^*)+4\Delta_{T-1}$ to approximately $R(h^{\mathcal{T}})+2\Delta_{T-1}$, and the label complexity upper bound can be decreased from $4 \theta\left(TR(h^{*})+2O(\sqrt{T})\right)$ to approximately $2\theta\left(2TR(h^{\mathcal{T}})+3 O(\sqrt{T})\right)$.


In conclusion, by improving hypothesis pruning, black-box teaching guides an active learner to converge into tighter bounds on generalization error and label complexity.


\section{Self-improvement of Teaching}\label{subsec:Self-improvement of Teaching}
Section~\ref{subsec:Teaching Improves hypothesis pruning} demonstrates that black-box teaching an active learner is effective. However, for Assumption~\ref{ass:T-AL}, if the teaching hypothesis is loosely approximated to the optimal hypothesis, i.e. $\epsilon$ is large, how do we guarantee the convergence of black-box teaching? We thus design a self-improvement of teaching strategy, which generates new hypotheses after each hypothesis pruning and determines whether to update the teacher. We then observe the improvement of teaching performance and further analyze gains for the active learner of the bounds on generalization error and label complexity.





 
\textbf{New hypotheses} Since hypothesis pruning is a process of shrinking the candidate hypothesis set, generating new hypotheses should not interrupt this process. More specifically, we require that the sequence of candidate hypothesis sets satisfy $\mathrm{Conv}(H^{\mathcal{T}}_{t+1}) \subseteq  \mathrm{Conv}(H^{\mathcal{T}}_{t})$, where $\mathrm{Conv}(\cdot)$ is the convex hull of a set. To avoid any confusion, we denote by $H'_t$ the candidate hypothesis set after pruning at $t$-time. Therefore, after pruning the hypothesis set from $H^{\mathcal{T}}_t$ to $H'_t$ at $t$-time, we generate new hypotheses $\tilde{h}$ from the convex hull of $H'_t$:
\begin{equation}\label{equ:new_hypothesis}
\begin{split}
&\tilde{h}=\sum_{j}^{m}\lambda_j h_j,h_j\in H'_{t},\\
\end{split}
\end{equation}
subjected to $ \sum_j^m \lambda_j=1,\lambda_j\in[0,1]$, where $m$ denotes the size of $H'_{t}$. We use Eq.~(\ref{equ:new_hypothesis}) to generate $n$ hypotheses for obtaining the hypothesis set $\widetilde{H}'_t = \left\{\tilde{h}_i;i\in[n]\right\}$ and combine it with $H'_{t}$ as the candidate hypothesis set next time: $H^{\mathcal{T}}_{t+1}=H'_{t}\cup \widetilde{H}'_t$.



\textbf{Self-improvement} The new hypotheses may perform better than the teaching hypothesis, that is, $R(\tilde{h})\leq R(h^{\mathcal{T}})$. By adding a restriction to the loss function, we give a condition for determining whether the teacher improves or not. We assume $\ell(h(x),y)=\phi\left(yh(x)\right)$, where  $\phi$ is functional non-increasing and convex. In short, $\ell(h(x),y)$ can be specified as 0-1, hinge, logistic loss functions, etc. Under the additional assumptions of the loss function, the following lemma reveals the variation of the maximum error disagreement in the candidate hypothesis set.





\begin{lemma}\label{lem:max_disagreement}
For any teaching-hypothesis-class $\mathcal{H}^{\mathcal{T}}$, teaching an active learner runs on $\mathcal{H}^{\mathcal{T}}$. If the loss function can be rewritten to form $\ell(h(x),y)=\phi(yh(x))$ and the function $\phi$ is non-increasing and convex, for any candidate hypothesis set $H^{\mathcal{T}}_t$ and for all $x\in\mathcal{X}$, the following equation holds:
\begin{equation*}
\max_{h,h'\in \mathrm{Conv}(H^{\mathcal{T}}_t)}\max_{y} \left|\ell(h(x),y)- \ell(h'(x),y)\right|=
\max_{h,h'\in H^{\mathcal{T}}_t}\max_{y} \left|\ell(h(x),y)- \ell(h'(x),y)\right|,
\end{equation*}
where $\mathrm{Conv}(H^{\mathcal{T}}_t)$ is the convex hull of the hypothesis set $H^{\mathcal{T}}_t$.
\end{lemma}
Lemma~\ref{lem:max_disagreement} shows that the maximum error disagreement at a certain fixed sample $x$ will not increase despite the learning algorithm generating new hypotheses in the convex hull of $H^{\mathcal{T}}_t$. Based on this property, Theorem~\ref{thm:BTAL-improvement} presents the lower bound analysis for the generalization error difference between the teaching hypothesis $h^{\mathcal{T}}$ and new hypotheses $\tilde{h}$.
\begin{theorem}\label{thm:BTAL-improvement}
For any teaching-hypothesis-class $\mathcal{H}^{\mathcal{T}}$, teaching an active learner runs on $\mathcal{H}^{\mathcal{T}}$, where the sequence of candidate hypothesis sets satisfies $\mathrm{Conv}(H^{\mathcal{T}}_{t+1})\subseteq\mathrm{Conv}(H^{\mathcal{T}}_{t})$ with $H^{\mathcal{T}}_1 = \mathcal{H}^{\mathcal{T}}$. For any $t\in\mathbb{N}^{+}$, given any $\delta > 0$, with a probability at least $1-\delta$, for any $\tilde{h}\in \widetilde{H}'_t$, the following inequality holds:
\begin{equation}\label{equ:judge}
R(h^{\mathcal{T}})-R(\tilde{h})\geq L_{t}(h^{\mathcal{T}})-L_{t}(\tilde{h})-\left(1+\mathcal{F}^{\mathcal{T}}(\tilde{h})\right)\Delta_{t}.
\end{equation}
\end{theorem}



Theorem~\ref{thm:BTAL-improvement} presents the determine condition of self-improvement: if $\beta^{(t)}_{i}= L_{t}(h^{\mathcal{T}})-L_{t}(\tilde{h}_i)-\left(1+\mathcal{F}^{\mathcal{T}}(\tilde{h}_{i})\right)\Delta_{t}>0$, i.e., $R(\tilde{h}_{i})<R(h^{\mathcal{T}})$, then the teaching hypothesis of $\mathcal{H}^{\mathcal{T}}$ is updated to $h^{\mathcal{T}}=\widetilde{h}_i$. Thus self-improvement of teaching strategy reduces generalization error of the teaching hypothesis without excessive additional calculations.

\textbf{Improvement of teaching performance} Self-improvement of teaching strategy obtain a teaching hypothesis sequence $\left\{h^{\mathcal{T}}_1,...,h^{\mathcal{T}}_T\right\}$, where $h^{\mathcal{T}}_t$ denote the optimal hypothesis in $\bigcup\limits_{k=1}^t H^{\mathcal{T}}_k$. Based on Theorem~\ref{thm:BTAL-improvement}, Corollary~\ref{cor:self-teacher} gives the improvement of teaching performance.
\begin{corollary}\label{cor:self-teacher}
For any teaching-hypothesis-class $\mathcal{H}^{\mathcal{T}}$, teaching an active learner runs on $\mathcal{H}^{\mathcal{T}}$. Let $\alpha_t=\max\{\max_{i}\beta^{(t)}_{i}, 0\}$ and $h^{\mathcal{T}}_1$ be the initial teaching hypothesis. If the self-improvement of teaching is applied, given any $\delta > 0$, with a probability at least $1-\delta$, for any $T\in\mathbb{N}^{+}$, there exists inequality $R(h^{\mathcal{T}}_T) \leq R(h^{\mathcal{T}}_1) - \sum_{t=1}^{T-1}\alpha_t$.
\end{corollary} 

Corollary~\ref{cor:self-teacher} guarantees that under high probability, self-improvement of teaching  can reduce the generalization error of the initial teaching hypothesis by at least $\sum_{t=1}^{T-1}\alpha_t$. Moreover, assuming $\epsilon$ is the initial approximation error of the teaching hypothesis to the optimal hypothesis, there exists inequality $R(h^{\mathcal{T}}_T) \leq R(h^{*}) +\epsilon_{T}$, where $\epsilon_T := \epsilon - \sum_{t=1}^{T-1}\alpha_t (T>1)$ and $\epsilon_1 = \epsilon$. Thus the self-improvement of teaching alleviates the loose approximation of the teaching hypothesis to the optimal hypothesis w.r.t. $\epsilon$ of Assumption~\ref{ass:T-AL}.

\textbf{Learning guarantees} With the improvement of the teacher, the improvement of the learner is natural. Recalling Theorem~\ref{thm:TAL-learning guarantees}, we then present the learning guarantees for the self-improvement of teaching. The primary motivation is to replace the pre-defined teaching hypothesis $h^{\mathcal{T}}$ by a teaching hypothesis sequence $\left\{h^{\mathcal{T}}_1,...,h^{\mathcal{T}}_T\right\}$. At $t$-time, we denote by $\mathcal{F}^{\mathcal{T}}_t(\cdot):= \mathcal{L}(h^{\mathcal{T}}_t,\cdot)$ the disagreement feedback with latest teaching hypothesis $h^{\mathcal{T}}_t$. Because the disagreement coefficient $\theta$ w.r.t. Eq.~(\ref{equ:error disagreement coefficient}) is defined based on the varying candidate hypothesis set $H^{\mathcal{T}}_t$, it varies with time $t$. To make the theoretical results more concise, we assume that $\theta$ is stable for smooth distribution and does not change dramatically as $H^{\mathcal{T}}_t$ changes. The learning guarantees of Theorem~\ref{thm:TAL-learning guarantees} are then re-derived.
\begin{theorem}\label{thm:BTAL-learning guarantees}
For any teaching-hypothesis-class $\mathcal{H}^{\mathcal{T}}$, teaching an active learner runs on $\mathcal{H}^{\mathcal{T}}$. If the self-improvement of teaching is applied, given any $\delta > 0$, with a probability at least $1-\delta$, for any $T\in\mathbb{N}^{+}$, the following holds: 1) for any $t\in[T]$, holds $h^{\mathcal{T}}_t\in H^{\mathcal{T}}_t$; 

2) the generalization error holds
\begin{equation*}
R(\widehat{h}_T) \leq R(h^{*})+\left(2+\mathcal{F}^{\mathcal{T}}_{T-1}(\widehat{h}_{T-1})+\mathcal{F}^{\mathcal{T}}_{T-1}(\widehat{h}_T)\right)\Delta_{T-1} + \epsilon_{T-1};
\end{equation*}
3) if the learning problem has disagreement coefficient $\theta$, the label complexity is at most
\[
\tau_T  \leq 2 \theta\left(2TR(h^{*})+\big(3+\mathcal{F}^{\mathcal{T}}_{T-1}(\widehat{h}_{T-1})\big) O(\sqrt{T})+2T\epsilon_{T-1}\right).\]
\end{theorem}
Theorem~\ref{thm:BTAL-learning guarantees} shows that the optimal hypothesis of $\bigcup\limits_{k=1}^t H^{\mathcal{T}}_k$ is maintained in the candidate hypothesis set with a high probability at any $t$-time. Recalling Corollary~\ref{cor:self-teacher}, there exists $\epsilon = \epsilon_1 \leq \epsilon_{T-1}$, which shows that self-improvement of teaching strategy can further reduce the generalization error and label complexity bounds of the learner w.r.t. Theorem~\ref{thm:TAL-learning guarantees}. Moreover, the improvement of the active learner is decided by the improvement of the black-box teacher.

In conclusion, by generating new hypotheses, self-improvement of teaching strategy tightens the approximation of the teaching hypothesis to the optimal hypothesis, which provides more favorable learning guarantees for an active learner.

\section{Black-box Teaching-based Active Learning}\label{sec:Black-box Teaching-based Active Learning}
Based on the theoretical results of Section~\ref{sec:Black-box Teaching}, we present the black-box teaching-based active learning algorithm (BTAL), which guides a white-box learner. To guide a black-box learner, we then extend BTAL into BTAL$^+$.

\subsection{Teaching a White-box Learner}\label{subsec:Teaching a White-box Learner}
We here consider the teaching for  a white-box learner who discloses its hypothesis class information to the teacher. In this setting, the learner prunes the teaching-hypothesis-class $\mathcal{H}^{\mathcal{T}}$ by querying the sample labels and finally outputs a desired hypothesis. In each round, BTAL includes three stages: 1) query, 2) hypothesis pruning, and  3) self-improvement. Its pseudo-code is presented in Algorithm~\ref{alg:BT-AL}.

\begin{algorithm}[t]
    \caption{BTAL($\mathcal{H}^{\mathcal{T}},h^{\mathcal{T}},T,n$)}
    \label{alg:BT-AL}
    \begin{algorithmic}[1]
    \STATE {\bfseries Initialize:} $H^{\mathcal{T}}_1 = \mathcal{H}^{\mathcal{T}}, h^{\mathcal{T}}_1 = h^{\mathcal{T}},  \widetilde{H}'_t = \emptyset$
    \FOR{$t\in[T]$}
    \STATE $p_t=\max_{h,h'\in H^{\mathcal{T}}_t}\max_{y} |\ell(h(x_t),y)- \ell(h'(x_t),y)|$
    \STATE $Q_t\in\{0,1\}$ with $Q_t\sim\mathcal{B}(1,p_t)$
    \IF{$Q_t=1$}
    \STATE $y_t \gets \mathrm{LABEL}(x_t)$
    \STATE $\widehat{h}_t=\underset{h\in H^{\mathcal{T}}_t}{\operatorname{argmin}}\ L_{t}(h)$
    \STATE ${H}'_t \gets H^{\mathcal{T}}_{t}$
    \FOR{$i\in[n]$}
    \STATE $\widetilde{h} \gets \mathrm{Conv}({H}'_t); \widetilde{H}'_t = \widetilde{H}'_t \cup \widetilde{h}$
    \STATE $h^{\mathcal{T}}_{t+1} \gets \widetilde{h}$
    \ENDFOR
    \STATE $H^{\mathcal{T}}_{t+1}=H'_{t}\cup \widetilde{H}'_t$
    \ENDIF
    \ENDFOR
    \STATE {\bfseries Return} $\widehat{h}_T$
\end{algorithmic}
\end{algorithm}

  


\textbf{Query} (Steps~3-6) On the setting of white-box learner, BTAL adopts a similar label query strategy as IWAL w.r.t. Eq.~(\ref{equ:IWAL_pt}), with a slightly different hypothesis class. At $t$-time, BTAL does a Bernoulli trial $Q_t$ with a success probability $p_t$:
\begin{equation}
p_t=\max_{h,h'\in H^{\mathcal{T}}_t}\max_{y} \left|\ell(h(x_t),y)- \ell(h'(x_t),y)\right|.
\end{equation}
If $Q_t=1$, the algorithm queries the label $y_t$ of $x_t$.

\textbf{Hypothesis pruning} (Step~7-8) At any $t$-time, BTAL maintains a candidate hypothesis set $H^{\mathcal{T}}_t$ with $H^{\mathcal{T}}_1=\mathcal{H}^{\mathcal{T}}$.  After querying the label, BTAL updates the current empirical optimal  hypothesis $\widehat{h}_t=\underset{h \in H^{\mathcal{T}}_t}{\operatorname{argmin}}\ L_{t}(h)$ w.r.t. Eq.~(\ref{equ:importance-weighted empirical error}). Then the algorithm prunes the candidate hypothesis set from $H^{\mathcal{T}}_t$ to $H'_{t}$ according to Eq.~(\ref{equ:hyp-pruning}). At $T$-time, BTAL returns the hypothesis $\widehat{h}_T$ as the final hypothesis output.

\textbf{Self-improvement} (Steps~9-13) After the hypothesis pruning, BTAL will generate new hypotheses to improve the performance of the teaching hypothesis. At $t$-time, BTAL generates new hypotheses $\tilde{h}_i$ from the convex hull of $H'_{t}$ according to Eq.~(\ref{equ:new_hypothesis}) and obtains the hypothesis set $\widetilde{H}'_t=\left\{\tilde{h}_i;i\in[n]\right\}$. Next, the algorithm updates the teaching hypothesis according to Eq.~(\ref{equ:judge}) and uses $H^{\mathcal{T}}_{t+1}=H'_{t}\cup \widetilde{H}'_t$ as the candidate hypothesis set at $t+1$-time.

\subsection{Teaching a Black-box Learner}\label{subsec:Teaching a Black-box Learner}
Here, we consider a more challenging problem: the learner is also a black-box who can not disclose its hypothesis class information to the teacher. In this setting, the teaching-hypothesis-class $\mathcal{H}^{\mathcal{T}}$ of learner is non-transparent. Therefore, the learner tries to converge to the optimal hypothesis from an initial hypothesis $\widehat{h}_0$ by incremental updates. We extend BTAL into BTAL$^+$ for teaching a black-box learner. In each round, BTAL$^+$ includes three stages: 1)  query, 2) hypothesis pruning, and  3) self-improvement. Its pseudo-code is presented in Algorithm~\ref{alg:BTAL+}.

\begin{algorithm}[htbp]
    \caption{BTAL$^+(h^{\mathcal{T}}, T, n)$}
    \label{alg:BTAL+}
    \begin{algorithmic}[1]
    \STATE {\bfseries Initialize:}Teacher $h^{\mathcal{T}}_1=h^{\mathcal{T}}$, Learner $\widehat{h}_0$
    \FOR{$t\in[T]$}
    \STATE $p_t=\max_y|\ell(h^{\mathcal{T}}_{t}(x),y)- \ell(\widehat{h}_{t-1}(x),y)|$
    \STATE $Q_t\in\{0,1\}$ with $Q_t\sim\mathcal{B}(1,p_t)$
    \IF{$Q_t=1$}
    \STATE $y_t \gets \mathrm{LABEL}(x_t)$
    \STATE $\widehat{h}_t=\underset{h\in H^{\mathcal{T}}_t}{\operatorname{argmin}}\ L_{t}(h)$
    \IF{$ L_{t-1}(\widehat{h}_t) > L_{t-1}(\widehat{h}_{t-1}) + \Phi_{t-1}$}
    \STATE $\widehat{h}_{t}=\widehat{h}_{t-1}$
    \ENDIF
    \FOR{$i\in[n]$}
    \STATE $\tilde{h}\gets \lambda h^{\mathcal{T}}_{t}+(1-\lambda)\widehat{h}_t$
    \STATE $h^{\mathcal{T}}_{t+1} \gets \tilde{h}$
    \ENDFOR
    \ENDIF
    \ENDFOR
    \STATE {\bfseries Return} $\widehat{h}_T$
\end{algorithmic}
\end{algorithm}


\textbf{Query}(Steps~3-6) In the setting of black-box learner, the maximum error disagreement of the candidate hypothesis set $H^{\mathcal{T}}_t$ cannot be obtained. We thus re-characterize query probability $p_t$ by the maximum error disagreement  between teacher and learner:
\begin{equation}\label{equ:BT-AL_pt}
p_t=\max_y\left|\ell(h^{\mathcal{T}}_{t}(x),y)- \ell(\widehat{h}_{t-1}(x),y)\right|.
\end{equation}
Formally, $p_t$ should be defined as  $p_t=\max_y\left|\ell(h^{\mathcal{T}}_{t}(x),y)- \ell(\widehat{h}_{t}(x),y)\right|$. Before the update on $\widehat{h}_{t}$ at $t$-time, $\widehat{h}_{t-1}$ is used to approximate  $\widehat{h}_{t}$. At $t$-time, BTAL$^+$ does a Bernoulli trial $Q_t$  with success probability $p_t$ to decide whether to query the label of $x_t$.

\textbf{Hypothesis pruning}(Steps~7-10) Since the teaching-hypothesis-class $\mathcal{H}^{\mathcal{T}}$ is non-transparent, hypothesis pruning is generalized as a constraint in incremental updates. Specifically, we present a backtracking approach to ensure that the current empirical optimal  hypothesis $\widehat{h}_t$ is maintained in the candidate hypothesis set $H^{\mathcal{T}}_t$. After the learner is updated as $\widehat{h}_t$ at $t$-time, we judge whether the hypothesis pruning rule of $\widehat{h}_t$ is satisfied at $t-1$-time by the following inequality:
\begin{equation}\label{equ:slack}
L_{t-1}(\widehat{h}_t)\leq L_{t-1}(\widehat{h}_{t-1}) + \Phi_{t-1},
\end{equation}
where $\Phi_{t-1}=\left(1+\mathcal{F}^{\mathcal{T}}_{t-1}(\widehat{h}_{t-1})\right)\Delta_{t-1}$ represents the slack term. If Eq.~(\ref{equ:slack}) does not satisfy, it means that $\widehat{h}_t$ has already been pruned, and we backtrack the hypothesis  $\widehat{h}_{t} = \widehat{h}_{t-1}$.  The backtracking approach forces the learner not to be updated too far for each update, which prevents the learner to be disordered when updating towards a subsequent hypothesis.

\textbf{Self-improvement}(Steps~11-14) Because the teaching-hypothesis-class $\mathcal{H}^{\mathcal{T}}$ of the learner is non-transparent, we cannot generate new hypotheses from the convex hull of the candidate hypothesis set. We suggest generating new hypotheses by a linear combination of the teaching hypothesis $h^{\mathcal{T}}_{t}$ and the current empirical optimal  hypothesis $\widehat{h}_t$:
\begin{equation}\label{equ:new_hypothesis_ts}
    \tilde{h}=\lambda h^{\mathcal{T}}_{t}+(1-\lambda)\widehat{h}_t.
\end{equation}
At $t$-time, BTAL$^+$ generates $n$ new hypotheses by Eq.(\ref{equ:new_hypothesis_ts}) and determines whether to update the teacher according to Eq.~(\ref{equ:judge}).



\section{Experiments}\label{sec:Experiments}
To demonstrate our teaching idea of Section~\ref{sec:Black-box Teaching}, we present the empirical studies for teaching-based hypothesis pruning of Section~\ref{subsec:Teaching Improves hypothesis pruning}, and the self-improvement of teaching of Section~\ref{subsec:Self-improvement of Teaching}. With their guarantees, we then present real-world studies for BTAL of Section~\ref{subsec:Teaching a White-box Learner} and BTAL$^+$ of Section~\ref{subsec:Teaching a Black-box Learner}.

\textbf{Dataset} We experimented with algorithms on 7 binary classification datasets: \textit{skin, shuttle, magic04, covtype, nomao, jm1 and mnist}. Table~\ref{tab:dataset} shows the summary statistics for all datasets used in our experiment. We denote by $N$ the number of samples, by $Dim$ the number of features, and $R$ is the relative size of the minority class. For the high-dimensional datasets (\textit{covtype, nomao, jm1}), we only keep the first 10 principal components of its original features. For the multi-class datasets (\textit{shuttle, covtype}), we set the majority class as positive classes and all the remaining classes as negative classes. For \textit{mnist} dataset, we set the digit 3 as the positive class and the digit 5 as the negative class. For all datasets, we normalize each feature to $[0,1]$.
\begin{table}[htbp]
\setlength{\tabcolsep}{18pt}
  \centering
  \caption{Dataset summary in experiments.}\label{tab:dataset}
    \begin{tabular}{cccc}
    \toprule
    Dataset & $N$     & $Dim$   & $R$ \\
    \midrule
    \textit{skin}  & 245,057  & 3     & 0.208 \\
    \textit{magic04} & 19,020  & 10    & 0.352 \\
    \textit{shuttle} & 43,500  & 9     & 0.216 \\
    \textit{covtype} & 581,012  & 54    & 0.488 \\
    \textit{nomao} & 34,465  & 118   & 0.286 \\
    \textit{jm1}   & 10,880  & 21    & 0.193 \\
    \textit{mnist}   & 11,552  & 784    & 0.469 \\
    \bottomrule
    \end{tabular}%
  \label{tab:addlabel}%
\end{table}%



\subsection{Empirical Studies}\label{sec:Empirical Studies}
We present the following empirical studies on six UCI binary classification datasets: 1) whether the teaching-based hypothesis pruning of BTAL can prune the candidate hypothesis set faster than hypothesis pruning of IWAL; 2) whether self-improvement of teaching strategy of BTAL can reduce the generalization error of teaching hypothesis.


\textbf{Setting} In our empirical studies, we randomly generate $10,000$ hyperplanes with bounded norms as the initial hypothesis class $\mathcal{H}^{\mathcal{T}}$ and set the teaching hypothesis  as  that hypothesis with the minimum empirical error from $\mathcal{H}^{\mathcal{T}}$. For a $Dim$-dimensional dataset, the sample can be described as $\vec{x}=(x_1,...,x_{Dim})$. Correspondingly, the generated hyperplanes are $Dim+1$-dimensional and can be parameterized as $\vec{w}=(w_1,...,w_{Dim},b)$, where $b$ is the bias term. Thus the prediction of the hypothesis is $h(x)=\sum_{n=1}^{Dim} w_n x_n +b$. For all $(x,y)\in \mathcal{X}\times\mathcal{Y}$, the loss function is written as $\ell(h(x),y)=\log\left(1+\exp\big(-yh(x)\big)\right)$, and  we use function $g\left(\ell(h(x),y)\right)=2/\left(1+\exp\big(-\ell(h(x),y)\big)\right)-1$ to normalize the output of $\ell(h(x),y)$ to $[0,1]$. The hypothesis pruning strategy of IWAL follows Section~\ref{subsec:Learning Algorithm}, and BTAL follows Section~\ref{subsec:Teaching Improves hypothesis pruning}. To reduce computation, we use $10\%$ of unlabeled samples of $\mathcal{X}$ to calculate approximately $\mathcal{L}(\cdot,\cdot)$ w.r.t. Eq.~(\ref{equ:disagreement}). For example, at $t$-time, for all $x\in S$, we solve for the disagreement feedback of teacher and learner by traversing the label $y$, where $S$ is the unlabeled data subset of $\mathcal{X}$ s.t. $|S|=10\% \times|\mathcal{X}|$. If the dataset is split into training set and test set, we use $10\%$ of training set for calculating $\mathcal{L}(\cdot,\cdot)$ approximately to prevent leakage of test set information. We repeat the empirical studies $20$ times on each dataset and collect the average results with standard error.

\textbf{Teaching-based hypothesis pruning} To analyze the hypothesis pruning performance of our teaching idea, we employ IWAL to compare our proposed BTAL in the specified $\mathcal{H}^{\mathcal{T}}$. The size  of the candidate hypothesis set written as $|H^{\mathcal{T}}_t|$ is generalized  as a feasible measure to show the pruning speed. We thus present the dynamic  change of $|H^{\mathcal{T}}_t|$ with the number of query labels (on $\log_2$ scale) in Figure~\ref{fig:exp_one}. Since  BTAL applies  a tighter hypothesis pruning slack term (w.r.t. Eq.~(\ref{equ:hyp-pruning})) under the guidance of the black-box teacher, its pruning speed is naturally faster than that of IWAL in terms of the $|H^{\mathcal{T}}_t|$.

\begin{figure}[htbp] 
	\centering 
	\vspace{-0.35cm} 
	\subfigtopskip=2pt
	\subfigbottomskip=2pt 
	\subfigcapskip=-5pt 
	\subfigure[\textit{skin}]{
		\includegraphics[width=0.31\textwidth]{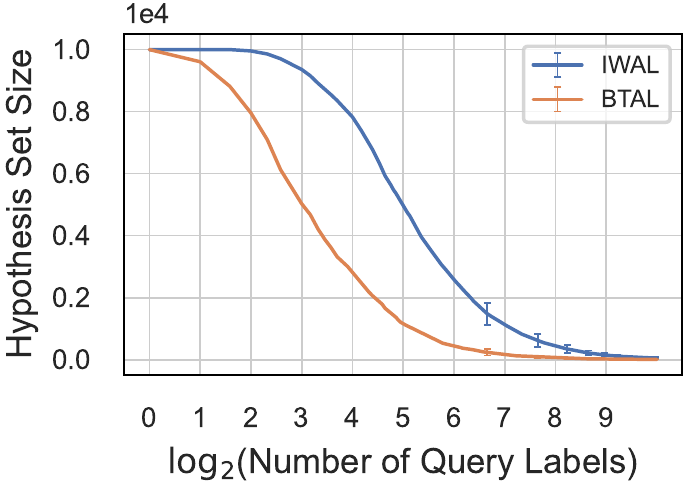}}
	\subfigure[\textit{magic04}]{
		\includegraphics[width=0.31\textwidth]{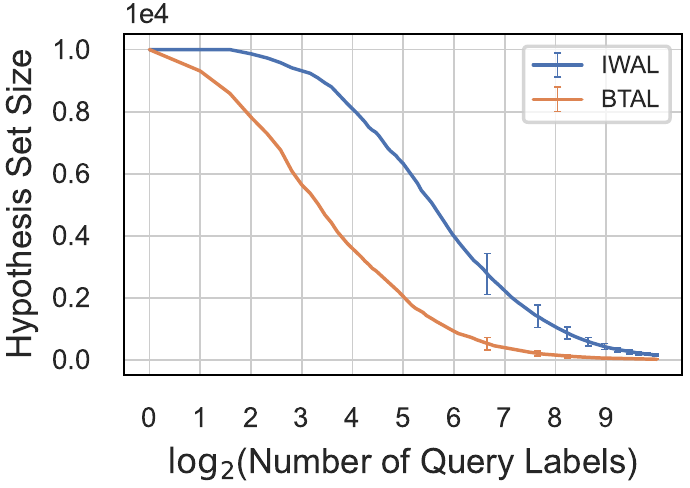}}
	\subfigure[\textit{shuttle}]{
		\includegraphics[width=0.31\textwidth]{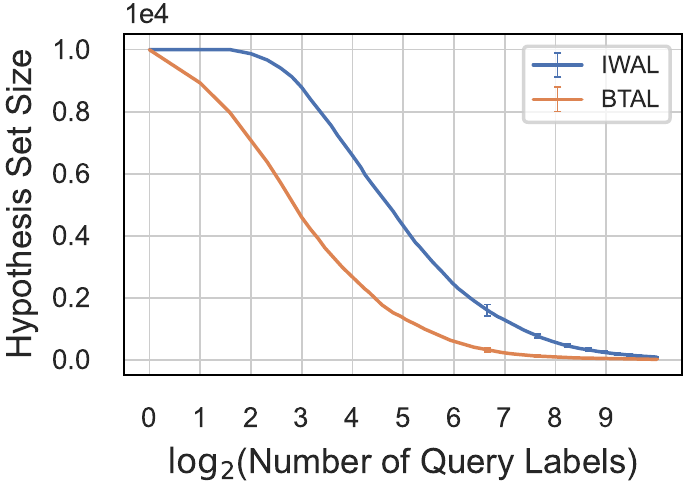}}
		
	\subfigure[\textit{covtype}]{
		\includegraphics[width=0.31\textwidth]{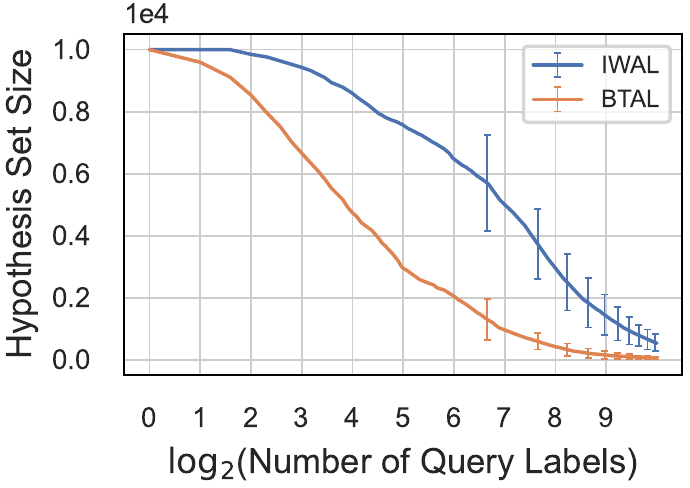}}
	\subfigure[\textit{nomao}]{
		\includegraphics[width=0.31\textwidth]{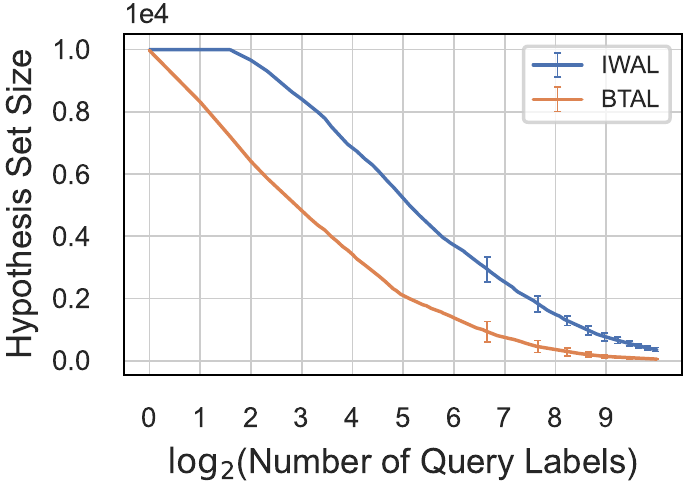}}
	\subfigure[\textit{jm1}]{
		\includegraphics[width=0.31\textwidth]{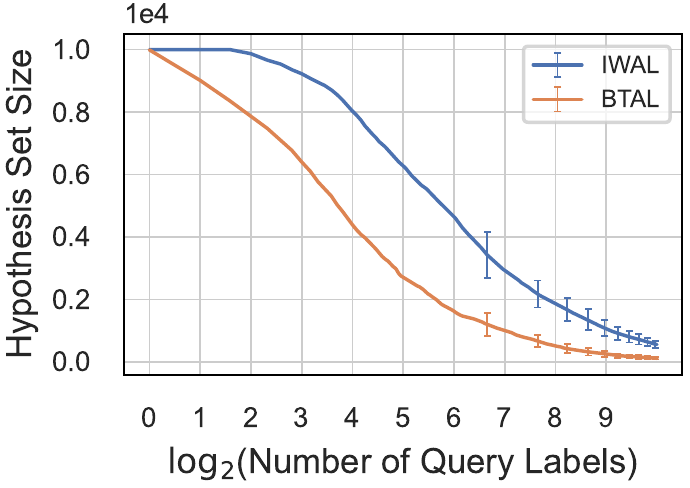}}
    \caption{The size of the candidate hypothesis set of IWAL and BTAL vs. the number of query labels ($\log2$ scale).}\label{fig:exp_one}
\end{figure}


\textbf{Self-improvement of teaching} To verify the effectiveness of the self-improvement of teaching, we observe changes in the generalization error of the teaching hypothesis for BTAL. For each dataset, we randomly select $50\%$ of the data as the training set and approximate the generalization error of the teaching hypothesis by the empirical error on the remaining data. The results are presented in Figure~\ref{fig:exp_two}. With self-improvement of teaching, BTAL gradually tightens the approximation of the teaching hypothesis to the optimal hypothesis. It then leads to the continuous and steady decreases in the generalization error curve of the teaching hypothesis for BTAL. 
\begin{figure}[htbp] 
	\centering 
	\vspace{-0.35cm} 
	\subfigtopskip=2pt
	\subfigbottomskip=2pt 
	\subfigcapskip=-5pt 
	\subfigure[\textit{skin}]{
		\includegraphics[width=0.31\textwidth]{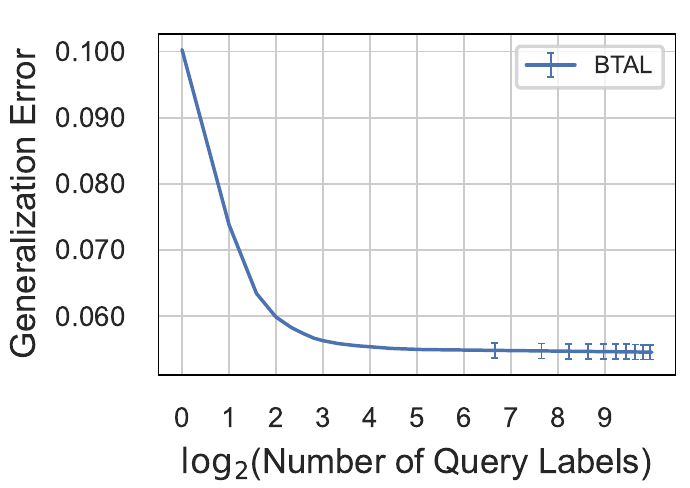}}
	\subfigure[\textit{magic04}]{
		\includegraphics[width=0.31\textwidth]{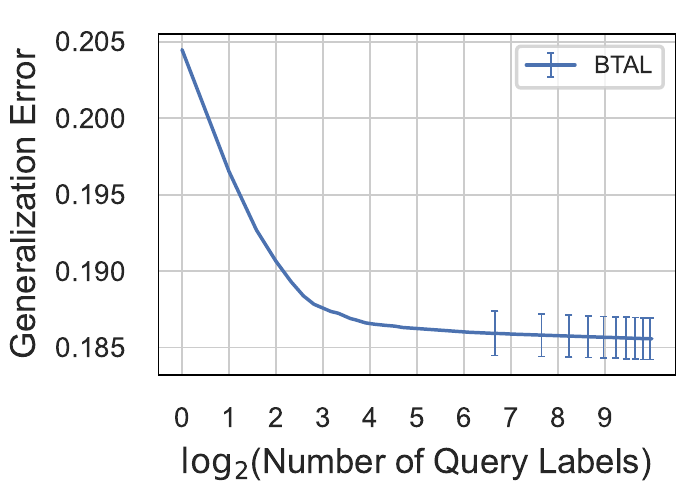}}
	\subfigure[\textit{shuttle}]{
		\includegraphics[width=0.31\textwidth]{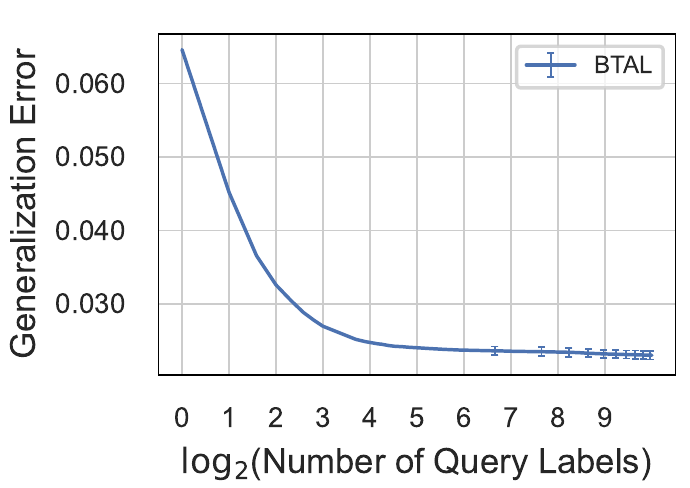}}
		
	\subfigure[\textit{covtype}]{
		\includegraphics[width=0.31\textwidth]{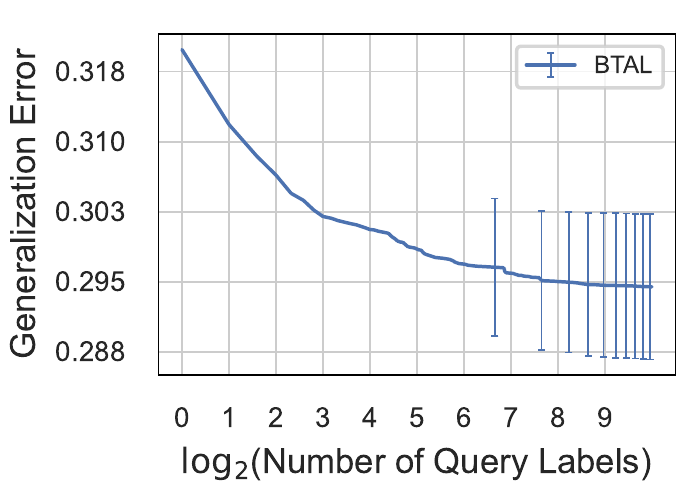}}
	\subfigure[\textit{nomao}]{
		\includegraphics[width=0.31\textwidth]{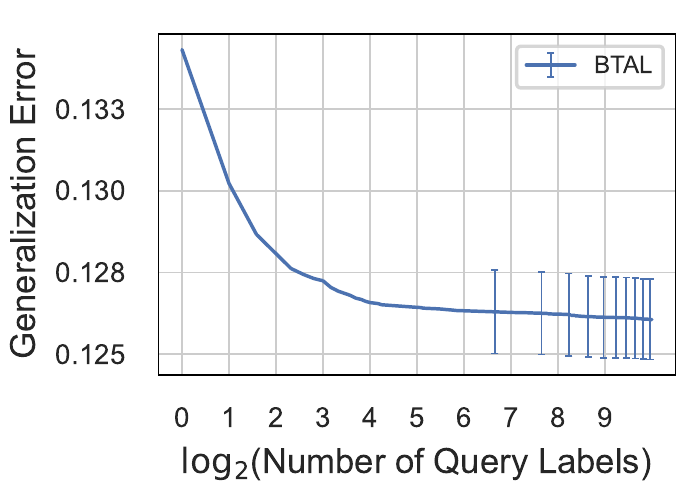}}
	\subfigure[\textit{jm1}]{
		\includegraphics[width=0.31\textwidth]{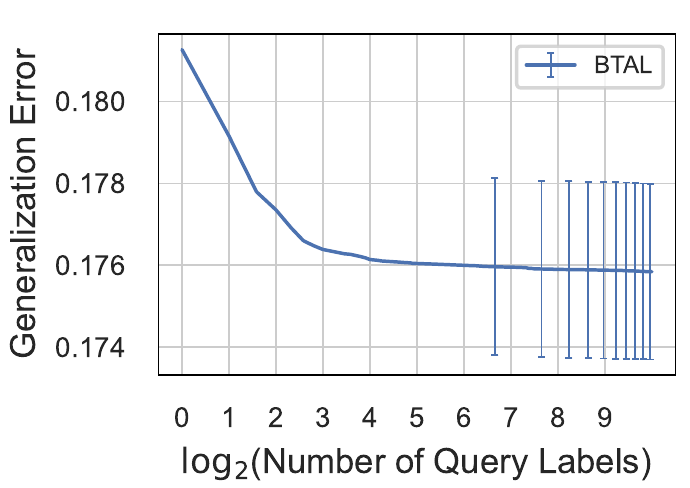}}
    \caption{The generalization error of teaching hypothesis for BTAL vs. the number of query labels ($\log2$ scale). }\label{fig:exp_two}
\end{figure}




\subsection{Real-world Studies}
We present the performance of BTAL and BTAL$^+$ in real-world studies. We first report the performance of BTAL in the setting of white-box learner, where IWAL\citep{beygelzimer2009importance} and IWAL-D\citep{cortes2019active} are used as the baseline. We then report the performance of BTAL$^+$ in the setting of black-box learner, where MVR\citep{freeman1965elementary}, ME\citep{shannon2001mathematical}, and Random\citep{gal2017deep} are used as the baseline. The reason for the different baselines in the two settings is that these algorithms are not directly applicable to each other.

\textbf{White-box learner} In this setting, we compare the performance of IWAL, IWAL-D, and BTAL on six UCI binary classification datasets. For all algorithms, we adopt the same settings as in Section~\ref{sec:Empirical Studies}, including 1) the initialization of the hypothesis set, 2) the loss function, and 3) the calculation method of $\mathcal{L}(\cdot,\cdot)$. For each dataset, we randomly select $70\%$ of the data as the training set and the remaining data  as the test set. We run the three algorithms $20$ times and collect the average results with standard error.

Firstly, we compare the performance of the hypothesis $\widehat{h}_T$ returned by IWAL, IWAL-D, and BTAL. Figure~\ref{fig:exp_three_a} presents the error rate of $\widehat{h}_T$ on the test dataset against the number of query labels (on $\log_2$ scale). The $\widehat{h}_T$ returned by IWAL and IWAL-D are subjected to the initial hypothesis class, so the error rate of $\widehat{h}_T$ is almost the same. However, the self-improvement of teaching strategy for BTAL can generate new hypotheses in the candidate hypothesis set, so $\widehat{h}_T$ has a lower error rate. This verifies the Theorem~\ref{thm:BTAL-learning guarantees} that the learner guided by a black-box teacher can converge into a tighter generalization error than those non-educated learners.

\begin{figure}[htbp] 
	\centering 
	\vspace{-0.35cm} 
	\subfigtopskip=2pt
	\subfigbottomskip=2pt 
	\subfigcapskip=-5pt 
	\subfigure[\textit{skin}]{
		\includegraphics[width=0.31\textwidth]{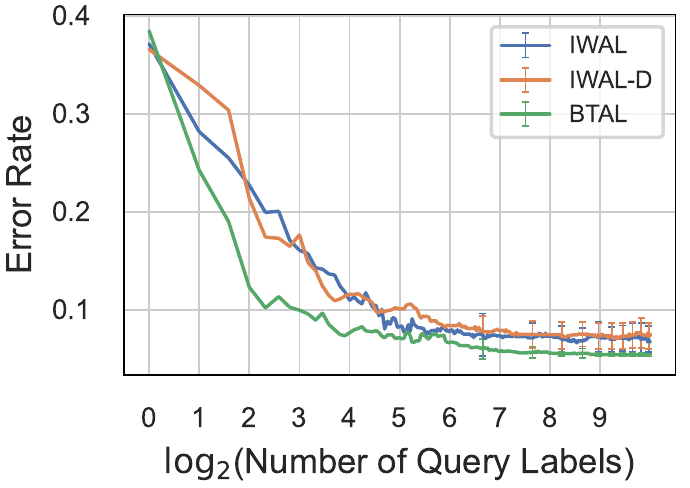}}
	\subfigure[\textit{magic04}]{
		\includegraphics[width=0.31\textwidth]{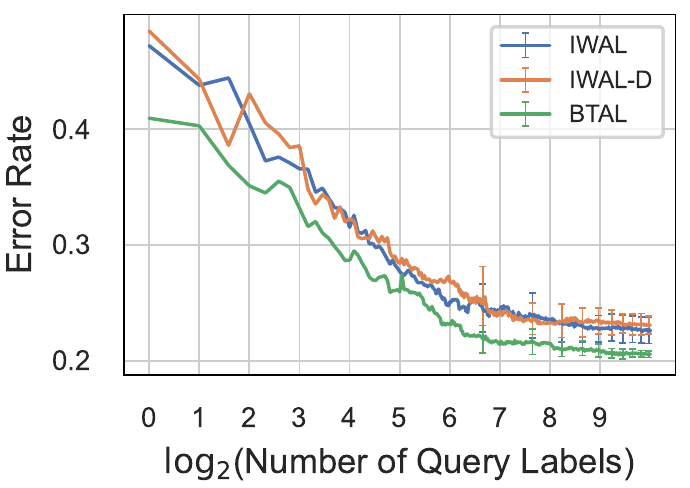}}
	\subfigure[\textit{shuttle}]{
		\includegraphics[width=0.31\textwidth]{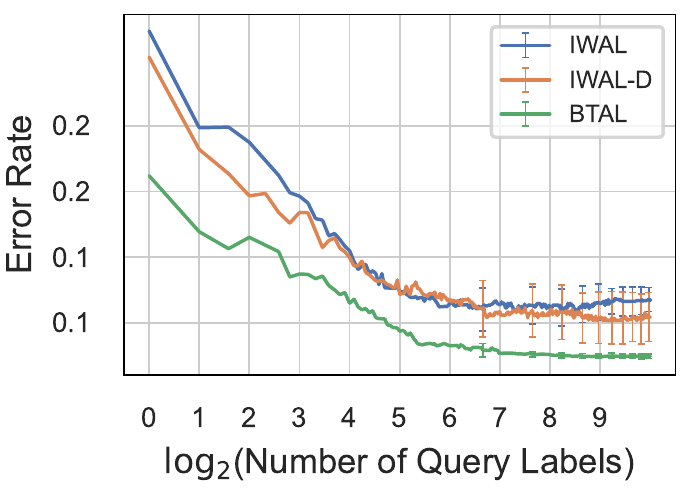}}
		
	\subfigure[\textit{covtype}]{
		\includegraphics[width=0.31\textwidth]{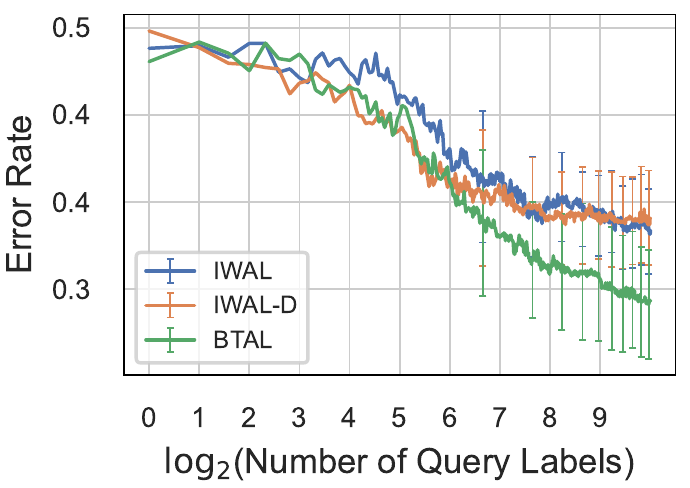}}
	\subfigure[\textit{nomao}]{
		\includegraphics[width=0.31\textwidth]{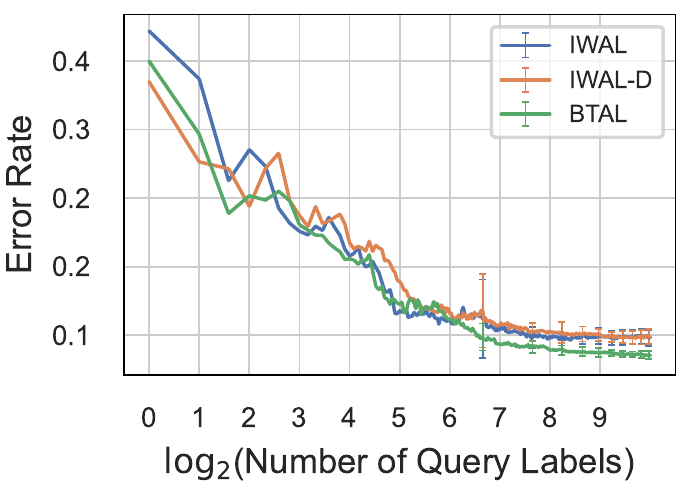}}
	\subfigure[\textit{jm1}]{
		\includegraphics[width=0.31\textwidth]{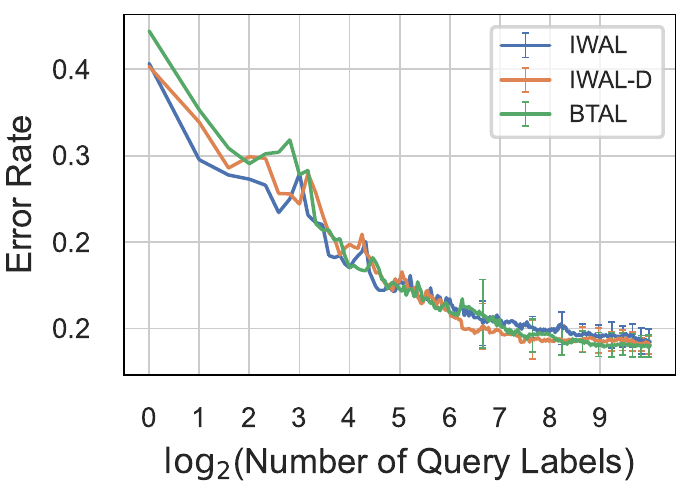}}
		\caption{The error rate of IWAL, IWAL-D, and BTAL on the test dataset vs. the number of query labels ($\log2$ scale).}\label{fig:exp_three_a}
\end{figure}



Secondly, we compare the number of query labels of IWAL, IWAL-D, and BTAL. Figure~\ref{fig:exp_three_b} presents the relationship between the number of query labels and the number of samples seen. IWAL-D uses the error disagreement of the learner for hypothesis pruning and thus spends  fewer the number of query labels than IWAL. BTAL uses the error disagreement between the teacher and the learner for hypothesis pruning, thus spending the fewest number of query labels. This further verifies the Theorem~\ref{thm:BTAL-learning guarantees} that the learner guided by a black-box teacher can converge into a tighter label complexity than  those non-educated  learners. 
\begin{figure}[htbp] 
	\centering 
	\vspace{-0.35cm} 
	\subfigtopskip=2pt
	\subfigbottomskip=2pt 
	\subfigcapskip=-5pt 
	\subfigure[\textit{skin}]{
		\includegraphics[width=0.31\textwidth]{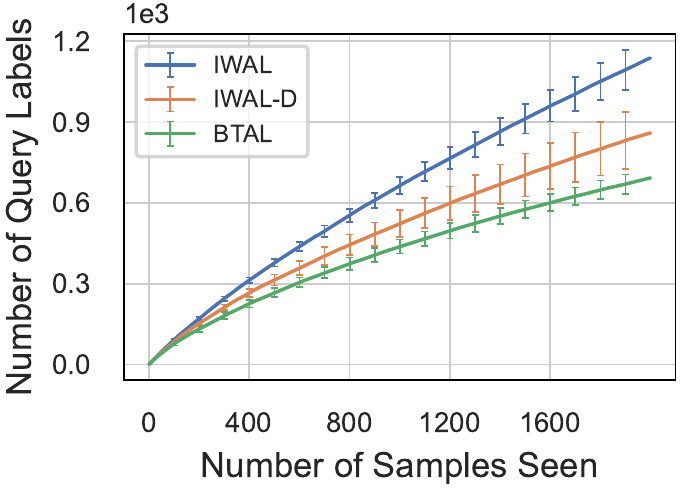}}
	\subfigure[\textit{magic04}]{
		\includegraphics[width=0.31\textwidth]{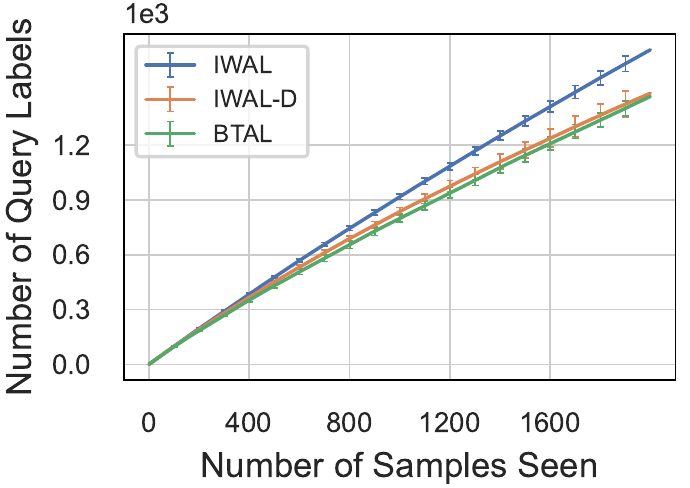}}
	\subfigure[\textit{shuttle}]{
		\includegraphics[width=0.31\textwidth]{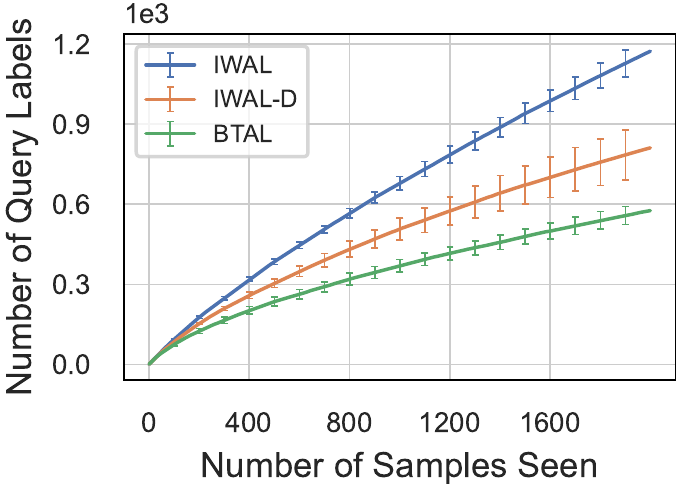}}
		
	\subfigure[\textit{covtype}]{
		\includegraphics[width=0.31\textwidth]{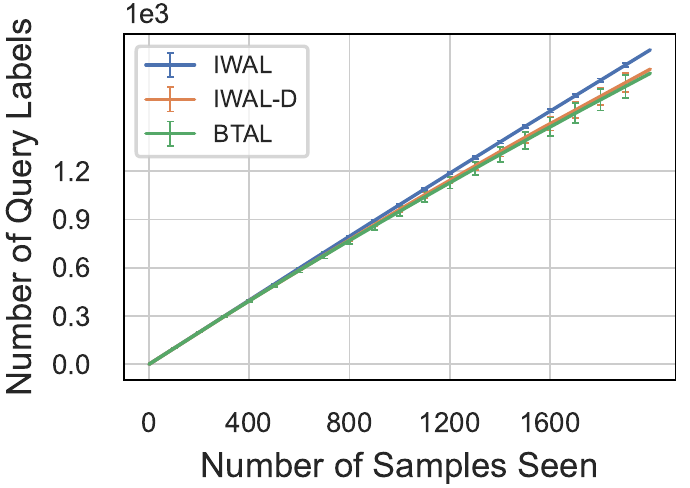}}
	\subfigure[\textit{nomao}]{
		\includegraphics[width=0.31\textwidth]{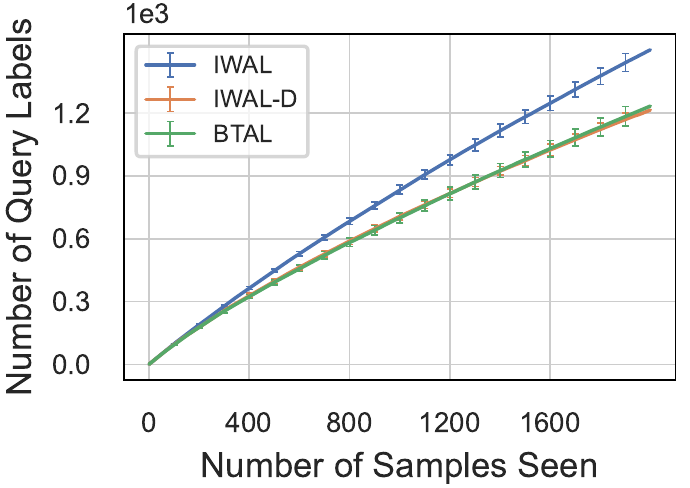}}
	\subfigure[\textit{jm1}]{
		\includegraphics[width=0.31\textwidth]{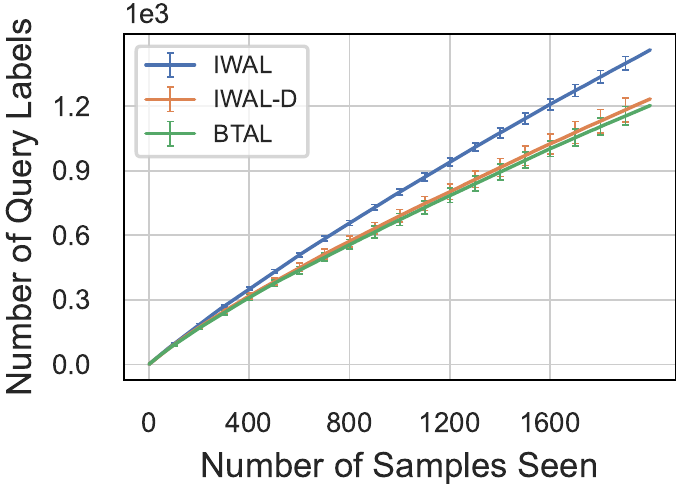}}
        \caption{The number of query labels of IWAL, IWAL-D, and BTAL vs. the number of samples seen.}\label{fig:exp_three_b}
\end{figure}

\begin{figure}[t]
    \centering
    \includegraphics[width=0.7\textwidth]{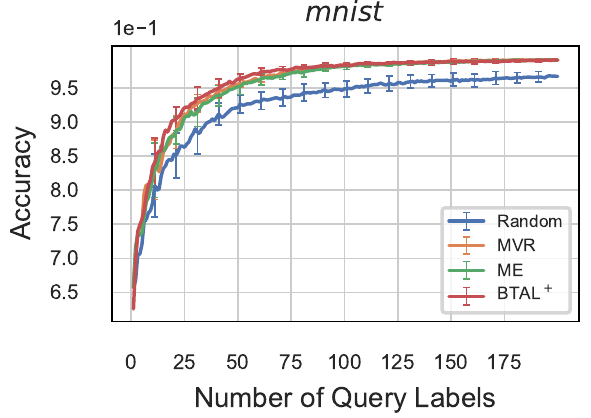}
    \vspace{-0.4cm}
    \caption{The accuracy of Random, MVR, ME, and BTAL$^+$ on the test dataset vs. the number of query labels.}\label{fig:exp_four}
    \vspace{-0.5cm}
\end{figure}

\textbf{Black-box learner} In this setting, we test BTAL$^+$ on the digits 3 and 5 of \textit{mnist} dataset\citep{crammer2009adaptive} as a binary classification task, where  $70\%$ of dataset is randomly selected as the training set and the remaining data as the test set. As a comparison, we examine the performance of several standard active learning algorithms: 1) maximize variation ratios (MVR), 2) max entropy (ME), and 3) random (Random). 

In all algorithms, we use the CNN network as a classifier following the structure of convolution-relu-convolution-relu-max pooling-dropout-dense-relu-dropout-dense, with the loss function: $\log(1+\exp(-yh(x))$ and normalize to $[0,1]$. In BTAL$^+$, the teaching hypothesis is specified as a pre-trained CNN model. 
We repeat the learning algorithm $20$ times and collect the average results with standard error. 

Figure~\ref{fig:exp_four} presents the relationship of  the test accuracy and the number of query labels, where BTAL$^+$  wins the traditional active learning baselines. The reasons are two-fold: 1) the traditional active learning algorithms have an unclear purpose, which leads to convergence of incremental updates that is usually infeasible; 2) BTAL$^+$ focuses on the disagreement between learner and teacher, and thus its convergence benefits from the pre-trained teaching hypothesis. This shows that teaching a black-box learner is also effective.


\section{Conclusion}\label{sec:Conclusion}
Black-box teaching an active learner is a new idea for the traditional active learning community. To keep fair teaching for error disagreement-based active learning, we set the machine teacher as a black-box, who only provides disagreement feedback to the learner. With this assumption, we introduce a teaching hypothesis to improve the hypothesis pruning, which results in tighter bounds on the generalization error and label complexity. Considering that the teaching hypothesis  may be a loose approximate 
to the optimal hypothesis, we also present the self-improvement of teaching. Guaranteed from our theoretical insights, we consider the teaching for a white-box and black-box learner. Rigorous analysis and solid experiments demonstrated the effectiveness of our teaching idea.
\newpage

\newpage

\appendix

\section{Proof}\label{app_proof}

\begingroup
\def\thetheorem{\ref{lem:T-AL}}
\begin{lemma}
For any teaching-hypothesis-class $\mathcal{H}^{\mathcal{T}}$, teaching an active learner runs on $\mathcal{H}^{\mathcal{T}}$, where the sequence of candidate hypothesis sets satisfies $H^{\mathcal{T}}_{t+1}\subseteq H^{\mathcal{T}}_{t}$ with $H^{\mathcal{T}}_1 = \mathcal{H}^{\mathcal{T}}$. Given any $\delta > 0$, with a probability at least $1-\delta$, for any $T\in\mathbb{N}^{+}$ and for all $h,h'\in H^{\mathcal{T}}_T$, the following inequality holds:
\begin{equation}\label{equ:LR inequality}
\left|L_T(h)-L_T(h')-\left(R(h)-R(h')\right)\right|\leq\left(1+\mathcal{L}(h,h')\right)\Delta_T,
\end{equation}
where $\Delta_{T}=\sqrt{(2/T)\log(2T(T+1)|\mathcal{H}^{\mathcal{T}}|^{2}/ \delta)}$.
\end{lemma}
\addtocounter{theorem}{-1}
\endgroup

\begin{proof}
Pick any $T\in\mathbb{N}^{+}$ and a pair of $h,h'\in H^{\mathcal{T}}_{T}$. We define a sequence of random variables $\{U_1,\cdots,U_{T}\}$, where $U_t(t\in[T])$ with respect to $h,h'$:
\begin{equation*}
U_t = \frac{Q_t}{p_t}\left[\ell(h(x_t),y_t)-\ell(h'(x_t),y_t)\right]-[R(h)-R(h')],
\end{equation*}
We then solve for the expectation of the random variable $U_t$ with respect to the past:
\begin{equation*}
\begin{split}
&\mathbb{E}[U_t|U_1,\cdots,U_{t-1}] \\ 
&=\mathop{\mathbb{E}}\limits_{(x_t,y_t)\sim \mathcal{D}}\frac{Q_t}{p_t}[\ell(h(x_t),y_t)-\ell(h'(x_t),y_t)]-[R(h)-R(h')]\\
& = 0.
\end{split}
\end{equation*}
This indicates that $U_t$ has zero expectation of the past, i.e., the sequence of random variables $\{U_1,\cdots,U_{T}\}$ is a martingale difference sequence. 

In order to use the Azuma’s inequality, we also need to prove the individual $U_t$ are bounded. We split $U_t$ into two parts and prove that each is bounded separately.

We first prove that $|\ell(h(x_t),y_t)-\ell(h'(x_t),y_t)|$ is bounded. For all hypothesis pruning strategies, the sequence of candidate hypothesis sets satisfies $H^{\mathcal{T}}_T\subseteq H^{\mathcal{T}}_{T-1}\subseteq \cdots \subseteq H^{\mathcal{T}}_1 = \mathcal{H}^{\mathcal{T}}$. Thus for all $t\leq T$, combine the definition of $p_t$, we have:
\begin{equation}\label{equ_app:pt}
\begin{split}
& |\ell(h(x_t),y_t)-\ell(h'(x_t),y_t)| \\ & \leq  \max_y \left|\ell(h(x_t),y)-\ell(h'(x_t),y)\right| \\ &  \leq \max_{h_1,h_2\in H^{\mathcal{T}}_T}\max_y  \left|\ell(h_1(x_t),y)-\ell(h_2(x_t),y)\right| \\ &  \leq \max_{h_1,h_2\in H^{\mathcal{T}}_t}\max_y  \left|\ell(h_1(x_t),y)-\ell(h_2(x_t),y)\right| \\ & = p_t.
\end{split}
\end{equation}
We next prove that $|R(h)-R(h')|$ is bounded.
\begin{equation*}
\begin{split}
|R(h)-R(h')| & =\left| \mathop{\mathbb{E}}\limits_{(x,y)\sim \mathcal{D}} [\ell(h(x),y)-\ell(h'(x),y)] \right| \\  & \leq  \mathop{\mathbb{E}}\limits_{(x,y)\sim \mathcal{D}} \left| \ell(h(x),y)-\ell(h'(x),y)\right| \\  &  \leq \mathop{\mathbb{E}}\limits_{x\sim \mathcal{D}_{\mathcal{X}}} \left[ \max_y |\ell(h(x),y)-\ell(h'(x),y)|\right] \\  & =\mathcal{L}(h,h').
\end{split}
\end{equation*}
Using the above inequality, we obtain that $|U_t|$ is bounded for all $t\in[T]$:
\begin{equation*}
\begin{split}
|U_t|  & = \left| \frac{Q_t}{p_t}\left[\ell(h(x_t),y_t)-\ell(h'(x_t),y_t)\right]-[R(h)-R(h')]\right| \\ & \leq  \frac{1}{p_t}\left|\ell(h(x_t),y_t) - \ell(h'(x_t),y_t)\right| + |R(h)-R(h')| \\ & \leq 1+\mathcal{L}(h,h').
\end{split}
\end{equation*}

Thus  $\{U_1,\cdots,U_{T}\}$ is a martingale difference sequence with bounded $1+\mathcal{L}(h,h')$. To make the subsequent proof clearer, let $Z_t=\frac{U_t}{1+\mathcal{L}(h,h')}$. Then $\{Z_1,\cdots,Z_{T}\}$ is a martingale difference sequence with bounded $|Z_t|\leq 1$. Applying Azuma’s inequality to $\sum_{t=1}^{T}Z_t$:

\begin{equation*}
\begin{split}
& \mathbb{P}\left(|L_T(h)-L_T(h')- R(h)+R(h')|\geq(1+\mathcal{L}(h,h'))\Delta_T\right) \\
& = \mathbb{P}\left(\frac{1}{T}|\sum_{t=1}^{T} Z_t|\geq \Delta_T \right) \\
& = \mathbb{P}\left(|\sum_{t=1}^{T} Z_t|\geq T \Delta_T \right) \\
& \leq 2 \exp(\frac{-T^2\Delta^{2}_T}{2T})\\
& = \frac{\delta}{T(T+1)|\mathcal{H}^{\mathcal{T}}|^2}.
\end{split}
\end{equation*}

The above probability inequality shows that the probability that Eq.~(\ref{equ:LR inequality}) does not hold is less than $\frac{\delta}{T(T+1)|\mathcal{H}^{\mathcal{T}}|^2}$.

Since $H^{\mathcal{T}}_T$ is a random subset of $\mathcal{H}^{\mathcal{T}}$, a union bound over all $T\in\mathbb{N}^{+}$ and all pairs of $h,h'\in H^{\mathcal{T}}_{T}$, we can concludes the proof.
\end{proof}

\begingroup
\def\thetheorem{\ref{thm:retain the teacher}}
\begin{theorem}
For any teaching-hypothesis-class $\mathcal{H}^{\mathcal{T}}$, teaching an active learner runs on $\mathcal{H}^{\mathcal{T}}$. Given any $\delta > 0$, with a probability at least $1-\delta$, for any $t\in\mathbb{N}^{+}$, the following inequality holds:
\begin{equation*}
L_t(h^{\mathcal{T}})-L_t(\widehat{h}_t)\leq\left(1+\mathcal{F}^{\mathcal{T}}(\widehat{h}_t)\right)\Delta_t.  
\end{equation*}
\end{theorem}
\addtocounter{theorem}{-1}
\endgroup
\begin{proof}
Start by assuming that the $1-\delta$ probability event of Lemma \ref{lem:T-AL} holds.

Let $t = T \in\mathbb{N}^{+}$. By using the absolute value inequality, we have:
\begin{equation*}
\begin{split}
& L_t(h^{\mathcal{T}})-L_t(\widehat{h}_t) \\
& \leq R(h^{\mathcal{T}}) - R(\widehat{h}_t) +(1+\mathcal{L}(h^{\mathcal{T}},\widehat{h}_t))\Delta_t \\
& \leq (1+\mathcal{F}^{\mathcal{T}}(\widehat{h}_t))\Delta_t.
\end{split}
\end{equation*}

The last inequality follows from the fact that $h^{\mathcal{T}}$ has the minimum generalization error in $\mathcal{H}^{\mathcal{T}}$, i.e., $R(h^{\mathcal{T}}) - R(\widehat{h}_t)\leq 0$.

From the arbitrariness of $T$, the theorem is proved.
\end{proof}

\begingroup
\def\thetheorem{\ref{thm:TAL-learning guarantees}}
\begin{theorem}
For any teaching-hypothesis-class $\mathcal{H}^{\mathcal{T}}$, teaching an active learner runs on $\mathcal{H}^{\mathcal{T}}$. Given any $\delta > 0$, with a probability at least $1-\delta$, for any $T\in\mathbb{N}^{+}$, the following holds: 

1) the generalization error holds \[R(\widehat{h}_T)\le R(h^{*})+\left(2+\mathcal{F}^{\mathcal{T}}(\widehat{h}_{T-1})+\mathcal{F}^{\mathcal{T}}(\widehat{h}_T)\right)\Delta_{T-1}+\epsilon;\]
2) if the learning problem has disagreement coefficient $\theta$, the label complexity is at most \[
\tau_T \leq 2 \theta\left(2TR(h^{*})+\big(3+\mathcal{F}^{\mathcal{T}}(\widehat{h}_{T-1}))\big) O(\sqrt{T})+2T\epsilon\right).\]
\end{theorem}
\addtocounter{theorem}{-1}
\endgroup

\begin{proof}
Start by assuming that the $1-\delta$ probability event of Lemma \ref{lem:T-AL} holds. 

Firstly, we give the bound of $R(\widehat{h}_T)$. Since $H^{\mathcal{T}}_T \subseteq H^{\mathcal{T}}_{T-1}$, there exists $\widehat{h}_T,h^{\mathcal{T}}\in H^{\mathcal{T}}_{T-1}$. To eliminate the importance-weighted empirical error, we consider Eq.~(\ref{equ:LR inequality}) with respect to $\widehat{h}_T,h^{\mathcal{T}}$ at $T-1$-time:
\begin{equation*}
\begin{split}
&   R(\widehat{h}_T) - R(h^{\mathcal{T}})\\
& \leq L_{T-1}(\widehat{h}_T) - L_{T-1}(h^{\mathcal{T}}) +(1+\mathcal{L}(\widehat{h}_T,h^{\mathcal{T}}))\Delta_{T-1} \\
& \leq L_{T-1}(\widehat{h}_{T-1})+ (1+ \mathcal{F}^{\mathcal{T}}(\widehat{h}_{T-1}))\Delta_{T-1} -  L_{T-1}(\widehat{h}_{T-1}) +(1+\mathcal{F}^{\mathcal{T}}(\widehat{h}_{T}))\Delta_{T-1} \\
& \leq \left(2+\mathcal{F}^{\mathcal{T}}(\widehat{h}_{T-1})+\mathcal{F}^{\mathcal{T}}(\widehat{h}_{T})\right)\Delta_{T-1},
\end{split}
\end{equation*}
where the second to last inequality follows from teaching-based hypothesis pruning rule w.r.t. Eq.~(\ref{equ:hyp-pruning}). Thus, for any $T\in\mathbb{N}^{+}$, the bound of generalization error for $\widehat{h}_T$ satisfies the following inequality:
\begin{equation*}
\begin{split}
R(\widehat{h}_T) & \le R(h^{\mathcal{T}})+\left(2+\mathcal{F}^{\mathcal{T}}(\widehat{h}_{T-1})+\mathcal{F}^{\mathcal{T}}(\widehat{h}_{T})\right)\Delta_{T-1}\\
& \le R(h^{*})+\left(2+\mathcal{F}^{\mathcal{T}}(\widehat{h}_{T-1})+\mathcal{F}^{\mathcal{T}}(\widehat{h}_T)\right)\Delta_{T-1}+\epsilon,
\end{split}
\end{equation*}
where the last inequality comes from Corollary~\ref{cor:error}.

Next, we give the upper bound of $\tau_T$. For any $h\in\mathcal{H}^{\mathcal{T}}$ and the teaching hypothesis $h^{\mathcal{T}}$, their disagreement $\rho(h,h^{\mathcal{T}})$ w.r.t Eq.~(\ref{equ:new disagreement}) has the upper bound:
\begin{equation*}
\begin{split}
\rho(h,h^{\mathcal{T}}) &  = \mathop{\mathbb{E}}\limits_{(x,y)\sim \mathcal{D}} \left| \ell(h(x),y)-\ell(h'(x),y)\right| \\
& \leq \mathop{\mathbb{E}}\limits_{(x,y)\sim \mathcal{D}} \left[ \ell(h(x),y)+\ell(h'(x),y)\right]\\
& \leq R(h)+R(h^{\mathcal{T}}).
\end{split}
\end{equation*}

For any $t\in[T]$, if $h\in H^{\mathcal{T}}_t$, using Lemma~\ref{lem:T-AL}, we have the upper bound for $R(h)$:
\begin{equation}\label{equ:rho of ht}
    R(h)\leq R(h^{\mathcal{T}})+\left(2+\mathcal{F}^{\mathcal{T}}(\widehat{h}_{t-1})+\mathcal{F}^{\mathcal{T}}(h)\right)\Delta_{t-1}.
\end{equation}
When $h\in H^{\mathcal{T}}_t$, we use Eq.~(\ref{equ:rho of ht}) to rewrite the upper bound of $\rho(h,h^{\mathcal{T}})$ as follows:
\begin{equation*}
\rho(h,h^{\mathcal{T}}) \leq 2R(h^{\mathcal{T}})+\left(2+\mathcal{F}^{\mathcal{T}}(\widehat{h}_{t-1})+\mathcal{F}^{\mathcal{T}}(h)\right)\Delta_{t-1}.
\end{equation*}
The above inequality shows that there is a common upper bound on the disagreement between any $h$ in $H^{\mathcal{T}}_t$ and the teaching hypothesis $h^{\mathcal{T}}$. Thus, we can construct a ball $B(h^{\mathcal{T}},r_t)$ such that $H^{\mathcal{T}}_t\subseteq B(h^{\mathcal{T}},r_t)$ for any $t$-time, where
\begin{equation}\label{equ:rt}
r_t=2R(h^{\mathcal{T}})+\left(2+\mathcal{F}^{\mathcal{T}}(\widehat{h}_{t-1})+\max_{h\in H^{\mathcal{T}}_t}\mathcal{F}^{\mathcal{T}}(h)\right)\Delta_{t-1}.
\end{equation}

Let $\mathcal{O}_t$ denote all the previous observations up to $t$-time: $\mathcal{O}_t=\left\{(x_1,y_1,p_1,Q_1),...,(x_t,y_t,p_t,Q_t)\right\}$ with $\mathcal{O}_0=\emptyset$. By using the error disagreement coefficient w.r.t Eq.~(\ref{equ:error disagreement coefficient}), the expected value of the query probability $p_t$ is at most:
\begin{equation*}
\begin{split}
&\mathop{\mathbb{E} }\limits_{x_t \sim  \mathcal{D}_{\mathcal{X}}}\left[p_t\mid\mathcal{O}_{t-1}\right]\\
& = \mathop{\mathbb{E} }\limits_{x_t \sim  \mathcal{D}_{\mathcal{X}}} \max_{h,h'\in H^{\mathcal{T}}_t}\max_{y} |\ell(h(x_t),y)- \ell(h'(x_t),y)| \\
& \leq 2 \mathop{\mathbb{E} }\limits_{x_t \sim  \mathcal{D}_{\mathcal{X}}} \max_{h\in H^{\mathcal{T}}_t}\max_{y} |\ell(h(x_t),y)- \ell(h^{\mathcal{T}}(x_t),y)| \\
& \leq 2 \mathop{\mathbb{E} }\limits_{x_t \sim  \mathcal{D}_{\mathcal{X}}} \max_{h\in B(h^{T},r_t)}\max_{y} |\ell(h(x_t),y)- \ell(h^{\mathcal{T}}(x_t),y)| \\
& \leq 2 \theta r_t \\
& = 2 \theta\left( 2R(h^{\mathcal{T}})+(2+\mathcal{F}^{\mathcal{T}}(\widehat{h}_{t-1})+\max_{h\in H^{\mathcal{T}}_t}\mathcal{F}^{\mathcal{T}}(h))\Delta_{t-1}\right),
\end{split}
\end{equation*}
where the first inequality follows from the triangle inequality, the second inequality follows from $H^{\mathcal{T}}_t\subseteq B(h^{\mathcal{T}},r_t)$, and the third inequality follows from the definition of $\theta$.

Summing over $t=1,...,T$, we get the upper bound of the label complexity $\tau_T$:
\begin{equation*}
\begin{split}
\tau_T &=  \sum_{t=1}^{T} \mathop{\mathbb{E} }\limits_{x_t \sim  \mathcal{D}_{\mathcal{X}}}\left[p_t\mid\mathcal{O}_{t-1}\right] \\
&= \sum_{t=1}^{T} 2 \theta\left( 2R(h^{\mathcal{T}})+(2+\mathcal{F}^{\mathcal{T}}(\widehat{h}_{t-1})+\max_{h\in H^{\mathcal{T}}_t}\mathcal{F}^{\mathcal{T}}(h))\Delta_{t-1}\right)\\
& \leq \sum_{t=1}^{T} 2 \theta\left(2R(h^{\mathcal{T}})+(3+\mathcal{F}^{\mathcal{T}}(\widehat{h}_{t-1}))\Delta_{t-1}\right)\\
& = 2 \theta\left(2TR(h^{\mathcal{T}})+(3+\mathcal{F}^{\mathcal{T}}(\widehat{h}_{T-1})) \sum_{t=1}^{T}  \Delta_{t-1}\right)\\
& \leq 2 \theta\left(2TR(h^{\mathcal{T}})+(3+\mathcal{F}^{\mathcal{T}}(\widehat{h}_{T-1})) O(\sqrt{T})\right),
\end{split}
\end{equation*}
where the last inequality uses $\sum_{t=1}^{T} \sqrt{\frac{1}{t}}=O(\sqrt{T})$. Recalling Corollary~\ref{cor:error}, there exists
\begin{equation*}
    \tau_T \leq 2 \theta\left(2TR(h^{*})+\big(3+\mathcal{F}^{\mathcal{T}}(\widehat{h}_{T-1}))\big) O(\sqrt{T})+2T\epsilon\right).
\end{equation*}

\end{proof}

\begingroup
\def\thetheorem{\ref{lem:max_disagreement}}
\begin{lemma}
For any teaching-hypothesis-class $\mathcal{H}^{\mathcal{T}}$, teaching an active learner runs on $\mathcal{H}^{\mathcal{T}}$. If the loss function can be rewritten to form $\ell(h(x),y)=\phi(yh(x))$ and the function $\phi$ is non-increasing and convex, for any candidate hypothesis set $H^{\mathcal{T}}_t$ and for all $x\in\mathcal{X}$, the following equation holds:
\begin{equation}\label{equ:left=right}
\max_{h,h'\in \mathrm{Conv}(H^{\mathcal{T}}_t)}\max_{y} \left|\ell(h(x),y)- \ell(h'(x),y)\right|=
\max_{h,h'\in H^{\mathcal{T}}_t}\max_{y} \left|\ell(h(x),y)- \ell(h'(x),y)\right|,
\end{equation}
where $\mathrm{Conv}(H^{\mathcal{T}}_t)$ is the convex hull of the hypothesis set $H^{\mathcal{T}}_t$.
\end{lemma}
\addtocounter{theorem}{-1}
\endgroup
\begin{proof}
Let $f(x)=\max_{h,h'\in \mathrm{Conv}(H^{\mathcal{T}}_t)}\max_{y} \left|\ell(h(x),y)- \ell(h'(x),y)\right|$ denote the left-hand side of Eq.~(\ref{equ:left=right}), and $g(x)=\max_{h,h'\in H^{\mathcal{T}}_t)}\max_{y} \left|\ell(h(x),y)- \ell(h'(x),y)\right|$ denote the right-hand side of Eq.~(\ref{equ:left=right}). For all $x\in\mathcal{X}$, we prove $f(x)\geq g(x)$, then prove $f(x)\leq g(x)$, and get $f(x)=g(x)$. Since $\mathrm{Conv}(H^{\mathcal{T}}_t) \supseteq H^{\mathcal{T}}_t$, there exists $f(x)\geq g(x)$. We next prove $f(x)\leq g(x)$.

For any $t\in[T]$ and for all hypotheses $h\in\mathrm{Conv}(H^{\mathcal{T}}_{t})$, $h$ can be linear representation by the hypothesis $h_j$ in $H^{\mathcal{T}}_t=\{h_1,h_2,...,h_{|H^{\mathcal{T}}_t|}\}$, that is
\begin{equation*}
    h=\sum_{j=1}^{|H^{\mathcal{T}}_t|} \lambda_j h_j,
\end{equation*}
where $\sum_j^m \lambda_j=1$ with $\lambda_j\in[0,1]$. 

Based on  the additional assumptions of the loss function, $\ell(h(x),y)=\phi(yh(x))$, where $\phi$ is a non-increasing function. Then, for all hypotheses $h\in\mathrm{Conv}(H^{\mathcal{T}}_{t})$ and for any $(x, y) \in \mathcal{X}\times\mathcal{Y}$, $\ell(h(x),y)$ has the following upper bound:
\begin{equation}\label{equ:loss upper bound}
\begin{split}
\ell(h(x),y) &=\ell\left(\sum_{j=1}^{|H^{\mathcal{T}}_t|} \lambda_j h_j(x),y\right) \\ 
&= \phi\left(y\sum_{j=1}^{|H^{\mathcal{T}}_t|} \lambda_j h_j(x)\right) \\ 
& \leq \max_{h_j \in H^{\mathcal{T}}_t} \phi(yh_j(x)) \\
& = \max_{h_j \in H^{\mathcal{T}}_t} \ell(h_j(x),y).
\end{split}
\end{equation}
Similarly, $\ell(h(x),y)$ has the following lower bound:
\begin{equation}\label{equ:loss lower bound}
\begin{split}
\ell(h(x),y) &=\ell\left(\sum_{j=1}^{|H^{\mathcal{T}}_t|} \lambda_j h_j(x),y\right) \\ 
&= \phi\left(y\sum_{j=1}^{|H^{\mathcal{T}}_t|} \lambda_j h_j(x)\right) \\ 
& \geq \min_{h_j \in H^{\mathcal{T}}_t} \phi(yh_j(x)) \\
& = \min_{h_j \in H^{\mathcal{T}}_t} \ell(h_j(x),y).
\end{split}
\end{equation}

With inequalities Eq.~(\ref{equ:loss upper bound})-(\ref{equ:loss lower bound}), for all hypotheses $h \in \mathrm{Conv}(H^{\mathcal{T}}_T)$ over any $(x, y) \in \mathcal{X}\times\mathcal{Y}$, the loss $\ell(h(x),y)$ can be bounded by the loss of hypothesis $h_j\in H^{\mathcal{T}}_t$.

Using the above properties of the loss function w.r.t. Eq.~(\ref{equ:loss upper bound})-(\ref{equ:loss lower bound}), we can derive the upper bound of error disagreement for any hypothesis pair $\{h,h'\}\subseteq \mathrm{Conv}(H^{\mathcal{T}}_t)$ over all $(x, y) \in \mathcal{X}\times\mathcal{Y}$. For any $h,h'\in \mathrm{Conv}(H^{\mathcal{T}}_t)$ and for all $(x, y) \in \mathcal{X}\times\mathcal{Y}$, there exists
\begin{equation*}
\begin{split}
&\ell(h(x),y)- \ell(h'(x),y) \\ 
&\leq \max_{h\in H^{\mathcal{T}}_t}\ell(h(x),y)- \min_{h\in H^{\mathcal{T}}_t} \ell(h(x),y)  \\ 
& \leq\max_{h,h'\in H^{\mathcal{T}}_t} |\ell(h(x),y)- \ell(h'(x),y)| \\
& \leq\max_{h,h'\in H^{\mathcal{T}}_t}\max_{y} |\ell(h(x),y)- \ell(h'(x),y)|.
\end{split}
\end{equation*}

The above inequality shows that the error disagreement in $\mathrm{Conv}(H^{\mathcal{T}}_t)$ over a fixed sample $x$ will not exceed the maximum error disagreement in $H^{\mathcal{T}}_t$ over  $x$.

Take the common upper bound on the left side of the inequality to conclude the proof.
\end{proof}

\begin{lemma}\label{lem:new T-AL}
For any teaching-hypothesis-class $\mathcal{H}^{\mathcal{T}}$, teaching an active learner runs on $\mathcal{H}^{\mathcal{T}}$, where the sequence of candidate hypothesis sets satisfies $\mathrm{Conv}(H^{\mathcal{T}}_{t+1})\subseteq\mathrm{Conv}(H^{\mathcal{T}}_{t})$ with $H^{\mathcal{T}}_1 = \mathcal{H}^{\mathcal{T}}$. Given any $\delta > 0$, with a probability at least $1-\delta$, for any $T\in\mathbb{N}^{+}$ and for all $h,h'\in \mathrm{Conv}(H^{\mathcal{T}}_T)$, the following inequality holds:
\begin{equation}
\left|L_T(h)-L_T(h')-\left(R(h)-R(h')\right)\right|\leq\left(1+\mathcal{L}(h,h')\right)\Delta_T,
\end{equation}
where $\Delta_{T}=\sqrt{(2/T)\log(2T(T+1)|\mathcal{H}^{\mathcal{T}}|^{2}/ \delta)}$.
\end{lemma}
\begin{proof}
Lemma~\ref{lem:new T-AL} complements the scenario of Lemma~\ref{lem:T-AL}, where the satisfying condition is weakened from $H^{\mathcal{T}}_{t+1}\subseteq H^{\mathcal{T}}_{t}$ to $\mathrm{Conv}(H^{\mathcal{T}}_{t+1})\subseteq\mathrm{Conv}(H^{\mathcal{T}}_{t})$. Recalling the process of proving Lemma~\ref{lem:T-AL}, to prove Lemma~\ref{lem:new T-AL}, we only need to show that Eq.~(\ref{equ_app:pt}) still holds.

For all $t\leq T$ and for any  $h,h'\in{\rm Conv}(H^{\mathcal{T}}_T)$, we have:
\begin{equation*}
\begin{split}
& |\ell(h(x_t),y_t)-\ell(h'(x_t),y_t)| \\ 
& \leq  \max_y \left|\ell(h(x_t),y)-\ell(h'(x_t),y)\right| \\ 
&  \leq \max_{h_1,h_2\in {\rm Conv}(H^{\mathcal{T}}_T)}\max_y  \left|\ell(h_1(x_t),y)-\ell(h_2(x_t),y)\right| \\
&  \leq \max_{h_1,h_2\in {\rm Conv}(H^{\mathcal{T}}_t)}\max_y  \left|\ell(h_1(x_t),y)-\ell(h_2(x_t),y)\right| \\
&  = \max_{h_1,h_2\in H^{\mathcal{T}}_t}\max_y  \left|\ell(h_1(x_t),y)-\ell(h_2(x_t),y)\right| \\ & = p_t,
\end{split}
\end{equation*}
where the second to last inequality follows that maximum error disagreement at a certain fixed sample $x$ will not increase w.r.t. Lemma~\ref{lem:max_disagreement}. 
\end{proof}

\begingroup
\def\thetheorem{\ref{thm:BTAL-improvement}}
\begin{theorem}
For any teaching-hypothesis-class $\mathcal{H}^{\mathcal{T}}$, teaching an active learner runs on $\mathcal{H}^{\mathcal{T}}$, where the sequence of candidate hypothesis sets satisfies $\mathrm{Conv}(H^{\mathcal{T}}_{t+1})\subseteq\mathrm{Conv}(H^{\mathcal{T}}_{t})$ with $H^{\mathcal{T}}_1 = \mathcal{H}^{\mathcal{T}}$. For any $t\in\mathbb{N}^{+}$, given any $\delta > 0$, with a probability at least $1-\delta$, for any $\tilde{h}\in \widetilde{H}'_t$, the following inequality holds:
\begin{equation*}
R(h^{\mathcal{T}})-R(\tilde{h})\geq L_{t}(h^{\mathcal{T}})-L_{t}(\tilde{h})-\left(1+\mathcal{F}^{\mathcal{T}}(\tilde{h})\right)\Delta_{t}.
\end{equation*}
\end{theorem}
\addtocounter{theorem}{-1}
\endgroup
\begin{proof}
Start by assuming that the $1-\delta$ probability event of Lemma \ref{lem:new T-AL} holds.

Let $t = T \in\mathbb{N}^{+}$. For all new hypotheses $\tilde{h}\in \widetilde{H}'_t$, there exists $\tilde{h}\in {\rm Conv}(H^{\mathcal{T}}_t)$ by the definition of $\tilde{h}$ w.r.t. Eq.~(\ref{equ:new_hypothesis}). At $t$-time, we use Lemma~\ref{lem:new T-AL} w.r.t. $h^{\mathcal{T}},\tilde{h}$ to obtain the following inequality:
\begin{equation*}
R(h^{\mathcal{T}})-R(\tilde{h})\geq L_{t}(h^{\mathcal{T}})-L_{t}(\tilde{h})-(1+\mathcal{F}^{\mathcal{T}}(\tilde{h}))\Delta_{t}.
\end{equation*}

From the arbitrariness of $T$, the theorem is proved.
\end{proof}

\begingroup
\def\thetheorem{\ref{thm:BTAL-learning guarantees}}
\begin{theorem}
For any teaching-hypothesis-class $\mathcal{H}^{\mathcal{T}}$, teaching an active learner runs on $\mathcal{H}^{\mathcal{T}}$. If the self-improvement of teaching is applied, given any $\delta > 0$, with a probability at least $1-\delta$, for any $T\in\mathbb{N}^{+}$, the following holds: 1) for any $t\in[T]$, holds $h^{\mathcal{T}}_t\in H^{\mathcal{T}}_t$; 

2) the generalization error holds
\begin{equation*}
R(\widehat{h}_T) \leq R(h^{*})+\left(2+\mathcal{F}^{\mathcal{T}}_{T-1}(\widehat{h}_{T-1})+\mathcal{F}^{\mathcal{T}}_{T-1}(\widehat{h}_T)\right)\Delta_{T-1} + \epsilon_{T-1};
\end{equation*}
3) if the learning problem has disagreement coefficient $\theta$, the label complexity is at most
\[
\tau_T  \leq 2 \theta\left(2TR(h^{*})+\big(3+\mathcal{F}^{\mathcal{T}}_{T-1}(\widehat{h}_{T-1})\big) O(\sqrt{T})+2T\epsilon_{T-1}\right).\]
\end{theorem}
\addtocounter{theorem}{-1}
\endgroup
\begin{proof}
Start by assuming that the $1-\delta$ probability event of Lemma~\ref{lem:new T-AL} holds. 

We first show that $h^{\mathcal{T}}_t\in  H^{\mathcal{T}}_t$  for any $t\in[T]$ by mathematical induction. It obviously applies to $t=1$. Now suppose it holds for $t=k$, that is, $h^{\mathcal{T}}_k\in H^{\mathcal{T}}_k$, let us prove that it is also true for $t=k+1$. By Lemma~\ref{lem:new T-AL}, there exists
\begin{equation*}
\begin{split}
& L_k(h^{\mathcal{T}}_k)-L_k(\widehat{h}_k) \\
& \leq R(h^{\mathcal{T}}_k) - R(\widehat{h}_k) +(1+\mathcal{L}(h^{\mathcal{T}}_k,\widehat{h}_k))\Delta_k \\
& \leq (1+ \mathcal{F}^{\mathcal{T}}_{k}(\widehat{h}_k))\Delta_k.
\end{split}
\end{equation*}
Therefore, we have $L_k(h^{\mathcal{T}}_k) \leq L_k(\widehat{h}_k)+(1+\mathcal{F}^{\mathcal{T}}_{k}(\widehat{h}_k))\Delta_k$, which shows that the teaching hypothesis $h^{\mathcal{T}}_k$ satisfies the hypothesis pruning rule, i.e., $h^{\mathcal{T}}_k\in H'_{k}$. If the self-improvement of teaching strategy does not find a better teaching hypothesis, then $h^{\mathcal{T}}_{k+1} = h^{\mathcal{T}}_{k} \in H'_{k}$. If the self-improvement of teaching strategy finds a better teaching hypothesis, then $h^{\mathcal{T}}_{k+1} \in \widetilde{H}'_{k}$. In both cases, there is always holds $h^{\mathcal{T}}_{k+1} \in H'_{k} \cup \widetilde{H}'_{k} = H^{\mathcal{T}}_{k+1}$. Thus $h^{\mathcal{T}}_t\in H^{\mathcal{T}}_{t}$ holds for any $t\in[T]$ by the mathematical induction.


Next, we give the bound of $R(\widehat{h}_T)$. Since $\widehat{h}_T \in H^{\mathcal{T}}_{T} =
H'_{T-1} \cup \widetilde{H}'_{T-1}$, we consider $h\in H'_{T-1}$ and $h\in \widetilde{H}'_{T-1}$ separately.

Assuming that $\widehat{h}_T \in H'_{T-1}$, we give an upper bound on the generalization error for any hypothesis $h$ in $H'_{T-1}$. Since $h\in H'_{T-1} \subseteq H^{\mathcal{T}}_{T-1}$, by Lemma~\ref{lem:new T-AL}, we have:
\begin{equation*}
\begin{split}
&   R(h) - R(h^{\mathcal{T}}_{T-1})\\
& \leq L_{T-1}(h) - L_{T-1}(h^{\mathcal{T}}_{T-1}) +(1+\mathcal{L}(h,h^{\mathcal{T}}_{T-1}))\Delta_{T-1} \\
& \leq L_{T-1}(\widehat{h}_{T-1})+ (1+\mathcal{F}^{\mathcal{T}}_{T-1}(\widehat{h}_{T-1}))\Delta_{T-1} -  L_{T-1}(\widehat{h}_{T-1}) +(1+\mathcal{F}^{\mathcal{T}}_{T-1}(\widehat{h}_{T}))\Delta_{T-1} \\
& \leq \left(2+\mathcal{F}^{\mathcal{T}}_{T-1}(\widehat{h}_{T-1})+\mathcal{F}^{\mathcal{T}}_{T-1}(\widehat{h}_{T})\right)\Delta_{T-1},
\end{split}
\end{equation*}
where $\mathcal{F}^{\mathcal{T}}_t(\cdot):= \mathcal{L}(h^{\mathcal{T}}_t,\cdot)$ denotes the disagreement feedback with latest teaching hypothesis $h^{\mathcal{T}}_t$ at $t$-time.

Assuming that $h\in \widetilde{H}'_{T-1}$, the second inequality above is no longer true because $h$ does not necessarily satisfy the hypothesis pruning rule. According to the self-improvement of teaching strategy w.r.t. Section~\ref{subsec:Self-improvement of Teaching}, we can express $h$ as $h=\sum_{j=1}^{|H'_{T-1}|} \lambda_j h_j$, where $h_j \in H'_{T-1}$ and $\sum_j^m \lambda_j=1$ with $\lambda_j\in[0,1]$. Based on the additional assumptions of the loss function, the following inequality portrays the upper bound of $R(h)$.
\begin{equation*}
\begin{split}
R(h) &=\mathop{\mathbb{E}}\limits_{(x,y) \sim  \mathcal{D}}[\ell(\sum_{j=1}^{|H'_{T-1}|} \lambda_j h_j(x),y)]  \\ 
&= \mathop{\mathbb{E}}\limits_{(x,y) \sim  \mathcal{D}}[\phi(y\sum_{j=1}^{|H'_{T-1}|} \lambda_j h_j(x))] \\ 
& \leq\sum_{j=1}^{|H'_{T-1}|} \lambda_j \mathop{\mathbb{E}}\limits_{(x,y)\sim  \mathcal{D}}[\phi(yh_j(x))] \\
& \leq \max_{h_j \in H'_{T-1}} \mathop{\mathbb{E}}\limits_{(x,y)\sim  \mathcal{D}}[\ell(h_j(x),y)]\\
& = \max_{h_j \in H'_{T-1}} R(h_j).
\end{split}
\end{equation*}
The above inequality shows that the generalization error of the hypothesis in $\widetilde{H}'_{T}$ will not be greater than the generalization error of the hypothesis in $H'_T$. 

Thus, for any $T\in\mathbb{N}^{+}$, the bound of generalization error for $\widehat{h}_T$ satisfies the following inequality:
\begin{equation*}
\begin{split}
 R(\widehat{h}_T)& \leq R(h^{\mathcal{T}}_{T-1})+\left(2+\mathcal{F}^{\mathcal{T}}_{T-1}(\widehat{h}_{T-1})+\mathcal{F}^{\mathcal{T}}_{T-1}(\widehat{h}_{T})\right)\Delta_{T-1}\\
 & \leq R(h^{*})+\left(2+\mathcal{F}^{\mathcal{T}}_{T-1}(\widehat{h}_{T-1})+\mathcal{F}^{\mathcal{T}}_{T-1}(\widehat{h}_T)\right)\Delta_{T-1} + \epsilon_{T-1},
\end{split}
\end{equation*}
where the last inequality comes from the definition of $\epsilon_{T-1}$ and Corollary~\ref{cor:self-teacher}.

The proof of label complexity is similar to Theorem~\ref{thm:TAL-learning guarantees}, which is omitted here.
\end{proof}

\bibliography{example_paper}

\end{document}